\documentclass{article}

\usepackage[utf8]{inputenc} %
\usepackage[T1]{fontenc}    %
\usepackage{amsthm}
\newtheorem{proposition}{Proposition}[section]
\newtheorem{lemma}{Lemma}[section]
\usepackage{url}            %
\usepackage{booktabs}       %
\usepackage{amsfonts}       %
\usepackage{nicefrac}       %
\usepackage{microtype}      %
\usepackage{xcolor}         %
\usepackage{fleqn, tabularx}
\usepackage{multirow}
\usepackage[export]{adjustbox}
\usepackage{amsmath}
\usepackage{colortbl}
\usepackage{wrapfig}
\usepackage{algorithm}
\usepackage[noend]{algorithmic}
\usepackage{xcolor}

\definecolor{Gray}{gray}{0.9}


\newcommand{\CC}{\cellcolor{Gray}}

\newcommand{\policy}{\pi}
\newcommand{\mdp}{\mathcal{M}}
\newcommand{\states}{\mathcal{S}}
\newcommand{\actions}{\mathcal{A}}

\newcommand{\behavior}{{\pi_\beta}}
\newcommand{\bellman}{\mathcal{B}}

\newcommand{\bx}{\mathbf{x}}

\newcommand{\bs}{\mathbf{s}}
\newcommand{\ba}{\mathbf{a}}

\newcommand{\methodname}{CDS}

\newcommand{\arxiv}[1] {{\color{black} #1}}

\usepackage{titlesec}
\titlespacing\section{0pt}{0pt plus 2pt minus 2pt}{0pt plus 2pt minus 2pt}
\titlespacing\subsection{0pt}{3pt plus 4pt minus 2pt}{0pt plus 2pt minus 2pt}
\titlespacing\subsubsection{0pt}{3pplus 4pt minus 2pt}{0pt plus 2pt minus 2pt}

\usepackage[colorlinks=true,linkcolor=blue,citecolor=blue]{hyperref}
\usepackage[skins,theorems]{tcolorbox}

\usepackage[numbers]{natbib}
\usepackage[preprint]{neurips_2021}

\title{Conservative Data Sharing \\ for Multi-Task Offline Reinforcement Learning}

\author{Tianhe Yu$^{*, 1, 2}$, Aviral Kumar$^{*, 2, 3}$, Yevgen Chebotar$^{2}$, Karol Hausman$^{1,2}$, \vspace{0.05cm}\\ \textbf{Sergey Levine$^{2,3}$, Chelsea Finn$^{1,2}$} \vspace{0.1cm}\\
$^1$Stanford University, $^2$Google Research, $^3$UC Berkeley~~~~~~~~ ($^*$Equal Contribution) \vspace{0.1cm}\\ 
\texttt{tianheyu@cs.stanford.edu, aviralk@berkeley.edu}
}

\begin{document}

\maketitle

\begin{abstract}
Offline reinforcement learning (RL) algorithms have shown promising results in domains where abundant pre-collected data is available. However, prior methods focus on solving individual problems from scratch with an offline dataset without considering how an offline RL agent can acquire multiple skills. \arxiv{We argue that a natural use case of offline RL is in settings where we can pool large amounts of data collected in various scenarios for solving different tasks, and utilize all of this data to learn behaviors for all the tasks more effectively rather than training each one in isolation. However, sharing data across all tasks in multi-task offline RL performs surprisingly poorly in practice.
Thorough empirical analysis, we find that sharing data can actually exacerbate the distributional shift between the learned policy and the dataset, which in turn can lead to divergence of the learned policy and poor performance.}
To address this challenge, we develop a simple technique for data-sharing in multi-task offline RL that routes data based on the improvement over the task-specific data. We call this approach conservative data sharing (\methodname), and it can be applied with multiple single-task offline RL methods.
On a range of challenging multi-task locomotion, navigation, and vision-based robotic manipulation problems, \methodname\ achieves the best or comparable performance compared to prior offline multi-task RL methods and previous data sharing approaches.  
\end{abstract}

\section{Introduction}

Recent advances in offline reinforcement learning (RL) make it possible to train policies for real-world scenarios, such as robotics~\citep{kalashnikov2018scalable,Rafailov2020LOMPO, kalashnikov2021mt} and healthcare~\citep{guez2008adaptive,shortreed2011informing,killian2020empirical}, entirely from previously collected data. Many realistic settings where we might want to apply offline RL are inherently \emph{multi-task} problems, where we want to solve multiple tasks using all of the data available. For example, if our goal is to enable robots to acquire a range of different behaviors, it is more practical to collect a modest amount of data for each desired behavior, resulting in a large but heterogeneous dataset, rather than requiring a large dataset for every individual skill. Indeed, many existing datasets in robotics~\citep{finn2017deep,dasari2020robonet,sharma2018multiple} and offline RL~\citep{fu2020d4rl} include data collected in precisely this way. Unfortunately, leveraging such heterogeneous datasets leaves us with two unenviable choices. We could train each task only on data collected for that task, but such small datasets may be inadequate for good performance. Alternatively, we could combine all of the data together \arxiv{and use data relabeled from other tasks to improve offline training}, but this na\"{i}ve data sharing approach can actually often degrade performance over simple single-task training in practice~\citep{kalashnikov2021mt}. 
In this paper, we aim to understand how data sharing affects RL performance in the offline setting and develop a reliable and effective method for selectively sharing data across tasks.

A number of prior works have studied multi-task RL in the \emph{online} setting, confirming that multi-tasking can often lead to performance that is worse than training tasks individually~\cite{parisotto2015actor,rusu2015policy,yu2020metaworld}. These prior works focus on mitigating optimization challenges that are aggravated by the online data generation process~\cite{schaul2019ray,yu2020gradient,yang2020multi}. As we will find in Section~\ref{sec:analysis}, multi-task RL remains a challenging problem in the offline setting when sharing data across tasks, even when exploration is not an issue. While prior works have developed heuristic methods for reweighting and relabeling data~\citep{andrychowicz2017hindsight,eysenbach2020rewriting, li2020generalized,kalashnikov2021mt}, they do not yet provide a principled explanation for why data sharing can hurt performance in the offline setting, nor do they provide a robust and general approach for selective data sharing that alleviates these issues while preserving the efficiency benefits of sharing experience across tasks.

In this paper, we hypothesize that data sharing can be harmful or brittle in the offline setting because it can exacerbate the distribution shift between the policy represented in the data and the policy being learned. %
We analyze the effect of data sharing in the offline multi-task RL setting, and present evidence to support this hypothesis.
Based on this analysis, we then propose an approach for selective data sharing that aims to minimize distributional shift, by sharing only data that is particularly relevant to each task. Instantiating a method based on this principle requires some care, since we do not know a priori which data is most relevant for a given task before we've learned a good policy for that task.
\arxiv{To provide a practical instantiation, we propose the conservative data sharing (CDS) algorithm. CDS reduces distributional shift by sharing data based on a learned conservative estimate of the Q-values that penalizes Q-values on out-of-distribution actions. Specifically, CDS relabels transitions when the conservative Q-value of the added transitions exceeds the expected conservative Q-values on the target task data. We visualize how \methodname\ works in Figure~\ref{fig:teaser}.}

\begin{wrapfigure}{r}{7cm}
    \vspace{-0.3cm}
    \centering
    \includegraphics[width=0.45\textwidth]{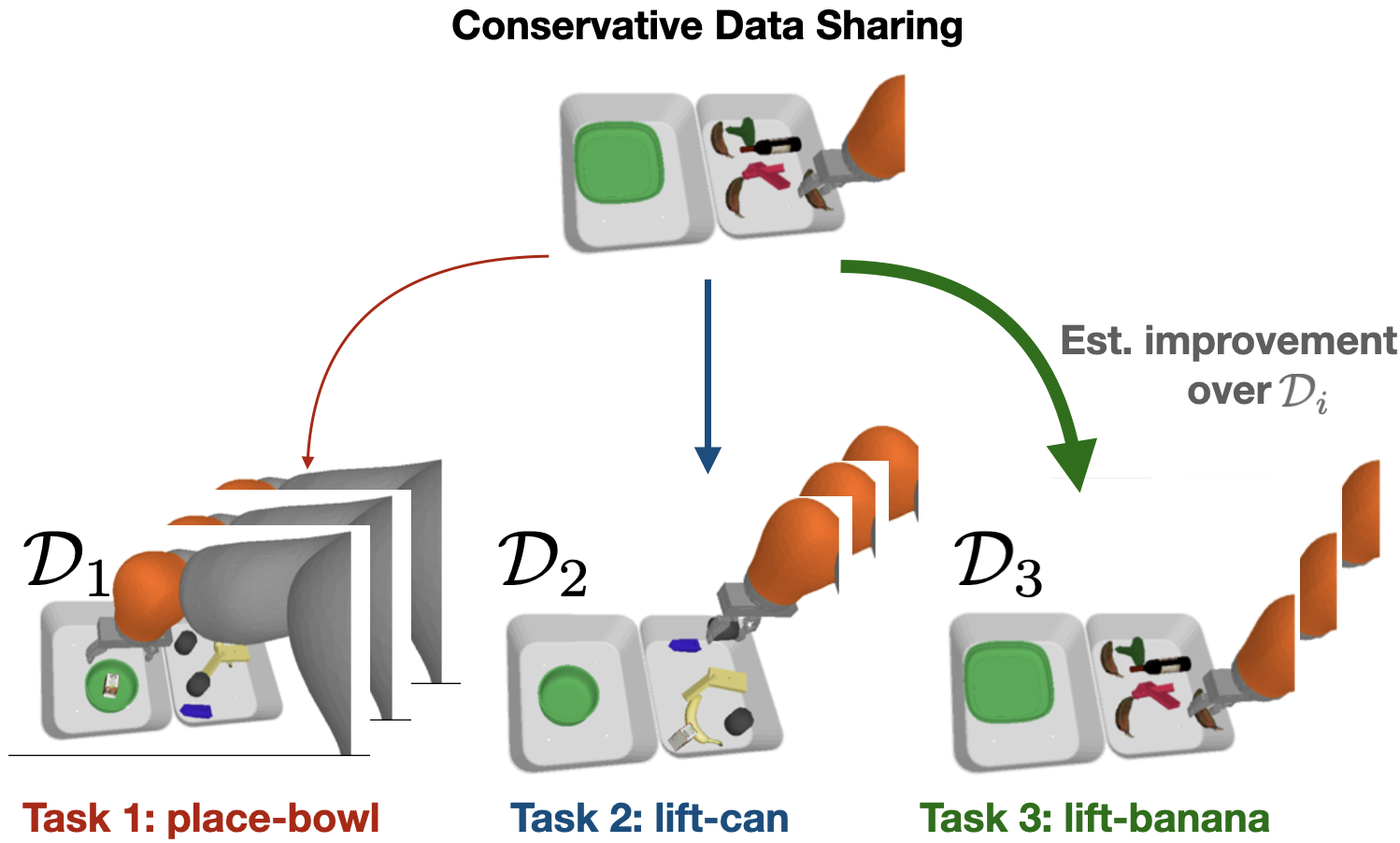}
    \vspace{-0.2cm}
    \caption{\footnotesize A visualization of \methodname, which routes a transition to
    the offline dataset $\mathcal{D}_i$ for each task $i$ with a weight based on the estimated improvement over the behavior policy $\pi_\beta(\ba|\bs, i)$ of $\mathcal{D}_i$ after sharing the transition.}
    \label{fig:teaser}
    \vspace{-0.3cm}
\end{wrapfigure}

The main contributions of this work are an analysis of data sharing in offline multi-task RL and a new algorithm, \textit{conservative data sharing} (\methodname), for multi-task offline RL problems. \arxiv{\methodname\ relabels a transition into a given task only when it is expected to improve performance based on a conservative estimate of the Q-function.} 
After data sharing, similarly to prior offline RL methods, \methodname\ applies a standard conservative offline RL algorithm, such as CQL~\citep{kumar2020conservative}, that learns a conservative value function or BRAC~\citep{wu2019behavior}, a policy-constraint offline RL algorithm.
Further, we theoretically analyze \methodname\ and characterize scenarios under which it provides safe policy improvement guarantees.
Finally, we conduct extensive empirical analysis of \methodname\ on multi-task locomotion, multi-task robotic manipulation with sparse rewards, multi-task navigation, and multi-task imaged-based robotic manipulation. We compare CDS to vanilla offline multi-task RL without sharing data, to na\"{i}vely sharing data for all tasks, and to existing data relabeling schemes for multi-task RL. \methodname\ is the only method to attain good performance across all of these benchmarks, often significantly outperforming the best \textit{domain-specific} method, improving over the next best method on each domain by \textbf{17.5\%} on average.

\section{Related Work}

\textbf{Offline RL.} Offline RL~\citep{ernst2005tree, riedmiller2005neural, LangeGR12, levine2020offline} has shown promise in domains such as robotic manipulation~\citep{kalashnikov2018scalable, mandlekar2020iris, Rafailov2020LOMPO,singh2020cog,kalashnikov2021mt}, NLP~\citep{jaques2019way,jaques2020human}, recommender systems \& advertising~\citep{strehl2010learning,garcin2014offline,charles2013counterfactual,theocharous2015ad,thomas2017predictive}, and healthcare~\citep{shortreed2011informing, Wang2018SupervisedRL}. The major challenge in offline RL is distribution shift~\citep{fujimoto2018off,kumar2019stabilizing,kumar2020conservative}, where the learned policy might generate out-of-distribution actions, resulting in erroneous value backups. Prior offline RL methods address this issue by regularizing the learned policy to be ``close`` to the behavior policy~\citep{fujimoto2018off,liu2020provably,jaques2019way,wu2019behavior, zhou2020plas,kumar2019stabilizing,siegel2020keep, peng2019advantage}, through variants of importance sampling~\citep{precup2001off, sutton2016emphatic, LiuSAB19, SwaminathanJ15, nachum2019algaedice}, via uncertainty quantification on Q-values~\citep{agarwal2020optimistic, kumar2019stabilizing, wu2019behavior, levine2020offline}, by learning conservative Q-functions~\citep{kumar2020conservative,kostrikov2021offline}, and with model-based training with a penalty on out-of-distribution states~\citep{kidambi2020morel, yu2020mopo,matsushima2020deployment,argenson2020model,swazinna2020overcoming,Rafailov2020LOMPO,lee2021representation,yu2021combo}. While current benchmarks in offline RL~\citep{fu2020d4rl,gulcehre2020rl} contain datasets that involve multi-task structure, existing offline RL methods do not leverage the shared structure of multiple tasks and instead train each individual task from scratch. In this paper, we exploit the shared structure in the offline multi-task setting and train a general policy that can acquire multiple skills.

\textbf{Multi-task RL algorithms.} Multi-task RL algorithms~\citep{wilson2007multi,parisotto2015actor,teh2017distral,espeholt2018impala,hessel2019popart,yu2020gradient, xu2020knowledge, yang2020multi, kalashnikov2021mt,sodhani2021multi}
focus on solving multiple tasks jointly in an efficient way. While multi-task RL methods seem to provide a promising way to build general-purpose agents~\citep{kalashnikov2021mt}, prior works have observed major challenges in multi-task RL, in particular, the optimization challenge~\citep{hessel2019popart,schaul2019ray,yu2020gradient}.
Beyond the optimization challenge, how to perform effective representation learning via weight sharing is another major challenge in multi-task RL. Prior works have considered distilling per-task policies into a single policy that solves all tasks~\citep{rusu2015policy,teh2017distral,ghosh2017divide,xu2020knowledge}, separate shared and task-specific modules with theoretical guarantees~\citep{d2019sharing}, and incorporating additional supervision~\citep{sodhani2021multi}. Finally, sharing data across tasks emerges as a challenge in multi-task RL, especially in the off-policy setting, as na\"{i}vely sharing data across all tasks turns out to hurt performance in certain scenarios~\citep{kalashnikov2021mt}. Unlike most of these prior works, we focus on the offline setting where the challenges in data sharing are most relevant. Methods that study optimization and representation learning issues are complementary and can be readily combined with our approach.

\textbf{Data sharing in multi-task RL.} Prior works~\citep{andrychowicz2017hindsight,kaelbling1993learning,pong2018temporal,schaul2015universal,eysenbach2020rewriting,li2020generalized,kalashnikov2021mt,chebotar2021actionable} have found it effective to reuse data across tasks by recomputing the rewards of data collected for one task and using such relabeled data for other tasks, which effectively augments the amount of data available for learning each task and boosts performance. These methods perform relabeling either uniformly~\citep{kalashnikov2021mt} or based on metrics such as estimated Q-values~\citep{eysenbach2020rewriting,li2020generalized}, domain knowledge~\citep{kalashnikov2021mt}, the distance to states or images in goal-conditioned settings~\citep{andrychowicz2017hindsight,pong2018temporal,nair2018visual,liu2019competitive,sun2019policy,lin2019reinforcement,huang2019mapping,lynch2020grounding,yang2021bias,chebotar2021actionable}, \arxiv{and metric learning for robust inference in the offline meta-RL setting~\citep{li2019multi}. All of these methods either require online data collection and do not consider data sharing in a fully offline setting, or only consider offline goal-conditioned or meta-RL problems~\citep{chebotar2021actionable,li2019multi}.} \arxiv{While these prior works empirically find that data sharing helps, we believe that our analysis in Section~\ref{sec:analysis} provides the first analytical understanding of why and when data sharing can help in multi-task offline RL and why it hurts in some cases.} 
\arxiv{Specifically, our analysis reveals the effect of distributional shift introduced during data sharing, which is not taken into account by these prior works. Our proposed approach, CDS, tackles the challenge of distributional shift in data sharing by intelligently sharing data across tasks and improves multi-task performance by effectively trading off between the benefits of data sharing and the harms of excessive distributional shift.}

\vspace{-5pt}
\section{Preliminaries and Problem Statement}
\vspace{-5pt}
\label{sec:prelims}

\textbf{Multi-task offline RL.} The goal in multi-task RL is to find a policy that maximizes expected return in a multi-task Markov decision process (MDP), defined as   $\mdp =(\states, \actions, P, \gamma, \{R_i, i \}_{i=1}^N)$, with state space $\states$, action space $\actions$, dynamics $P(\bs' | \bs, \ba)$, a discount factor $\gamma \in [0, 1)$, and a finite set of task indices $1, \cdots, N$
with corresponding reward functions $R_1, \cdots, R_N$. Each task $i$ presents a different reward function $R_i$, but we assume that the dynamics $P$ are shared across tasks. While this setting is not fully general, there are a wide variety of practical problem settings for which only the reward changes including various goal navigation tasks~\cite{fu2020d4rl}, distinct object manipulation objectives~\cite{xie2018few}, and different user preferences~\cite{christiano2017deep}.
In this work, we focus on learning a policy $\pi(\ba|\bs, i)$, which in practice could be modelled as independent policies $\{\pi_1(\ba|\bs), \cdots, \pi_N(\ba|\bs)\}$ that do not share any parameters, or as a single task-conditioned policy, $\pi(\ba|\bs, i)$ with parameter sharing. Our goal in this paper is to analyze and devise methods for data sharing and the choice of parameter sharing is orthogonal, and can be made independently.
We formulate the policy optimization problem as finding a policy that maximizes expected return over all the tasks: $\pi^*(\ba|\bs, \cdot) := \arg \max_{\pi} \mathbb{E}_{i \sim [N]} \mathbb{E}_{\pi(\cdot|\cdot, i)}[\sum_{t} \gamma^t R_i(\bs_t, \ba_t)]$.
The Q-function, $Q^\pi(\bs, \ba, i)$, of a policy $\pi(\cdot|\cdot, i)$ is the long-term discounted reward obtained in task $i$ by executing action $\ba$ at state $\bs$ and following policy $\pi$ thereafter.

Standard offline RL is concerned with learning policies $\pi(\ba|\bs)$ using only a given static dataset of transitions $\mathcal{D} =  \{(\bs_j, \ba_j, \bs'_j, r_j)\}_{j=1}^N$, collected by a behavior policy $\pi_\beta(\ba|\bs)$, without any additional environment interaction. In the multi-task offline RL setting, the dataset $\mathcal{D}$ is partitioned into per-task subsets, $\mathcal{D} = \cup_{i=1}^N \mathcal{D}_i$,
where $\mathcal{D}_i$ consists of experience from task $i$. While algorithms can choose to train the policy for task $i$ (i.e., $\pi(\cdot|\cdot, i)$) only on $\mathcal{D}_{i}$, in this paper, we are interested in data-sharing schemes that correspond to relabeling data from a different task, $j \neq i$ with the reward function $r_i$, and learn $\pi(\cdot|\cdot, i)$ on the combined data. To be able to do so, we assume access to the functional form of the reward $r_i$, a common assumption in goal-conditioned RL~\cite{andrychowicz2017hindsight,eysenbach2020rewriting}, and which often holds in robotics applications through the use of learned classifiers~\cite{xie2018few,kalashnikov2018scalable}, and discriminators~\cite{fu2018variational,chen2021learning}.

We assume that relabeling data $\mathcal{D}_j$ from task $j$ to task $i$ generates a dataset $\mathcal{D}_{j \rightarrow i}$, which is then additionally used to train on task $i$. Thus, the effective dataset for task $i$ after relabeling is given by $\mathcal{D}^\mathrm{eff}_i := \mathcal{D}_i \cup \left( \cup_{j \neq i} \mathcal{D}_{j \rightarrow i} \right)$. This notation simply formalizes data sharing and relabeling strategies explored in prior work~\citep{eysenbach2020rewriting,kalashnikov2021mt}. Our aim in this paper will be to improve on this na\"{i}ve strategy, which we will show leads to significantly better results. 

\textbf{Offline RL algorithms.} A central challenge in offline RL is distributional shift: differences between the learned policy and the behavior policy can lead to erroneous target values, where the Q-function is queried at actions $\ba \sim \pi(\ba|\bs)$ that are far from the actions it is trained on, leading to massive overestimation~\citep{levine2020offline,kumar2019stabilizing}.
A number of offline RL algorithms use some kind of regularization on either the policy~\citep{kumar2019stabilizing,fujimoto2018off,wu2019behavior,jaques2019way,siegel2020keep,peng2019advantage} or on the learned Q-function~\citep{kumar2020conservative,kostrikov2021offline} to ensure that the learned policy does not deviate too far from the behavior policy. For our analysis in this work, we will abstract these algorithms into a generic constrained policy optimization problem~\citep{kumar2020conservative}:
\begin{equation}
\label{eqn:generic_offline_rl}
    \pi^*(\ba|\bs) := \arg \max_{\pi}~~ J_{\mathcal{D}}(\pi) - \alpha D(\pi, \pi_\beta).
\end{equation}
$J_{\mathcal{D}}(\pi)$ denotes the average return of policy $\pi$ in the empirical MDP induced by the transitions in the dataset, and $D(\pi, \pi_\beta)$ denotes a divergence measure (e.g., KL-divergence~\citep{jaques2019way,wu2019behavior}, MMD distance~\citep{kumar2019stabilizing} or $D_{\text{CQL}}$~\citep{kumar2020conservative}) between the learned policy $\pi$ and the behavior policy $\pi_\beta$. 
In the multi-task offline RL setting with data-sharing, the generic optimization problem in Equation~\ref{eqn:generic_offline_rl} for a task $i$ utilizes the effective dataset $\mathcal{D}^\mathrm{eff}_i$. In addition, we define $\pi_\beta^\mathrm{eff}(\ba|\bs, i)$ as the effective behavior policy for task $i$ and it is given by: $\pi_\beta^\mathrm{eff}(\ba|\bs, i) := |\mathcal{D}^\mathrm{eff}_i(\bs, \ba)| / |\mathcal{D}^\mathrm{eff}_i(\bs)|$. Hence, the counterpart of Equation~\ref{eqn:generic_offline_rl} in the multi-task offline RL setting with data sharing is given by:
\begin{equation}
\label{eqn:generic_multitask_offline_rl}
     \forall i \in [N], ~~\pi^*(\ba|\bs, i) := \arg \max_{\pi}~~ J_{\mathcal{D}^\mathrm{eff}_i}(\pi) - \alpha D(\pi, \pi^\mathrm{eff}_\beta).
\end{equation}
We will utilize this generic optimization problem to motivate our method in Section~\ref{sec:method}.

\section{When Does Data Sharing Actually Help in Offline Multi-Task RL?}
\label{sec:analysis}

Our goal is to leverage experience from all tasks to learn a policy for a particular task of interest. Perhaps the simplest approach to leveraging experience across tasks is to train the task policy on not just the data coming from that task, but also relabeled data from all other tasks~\citep{caruana1997multitask}.
Is this na\"ive data sharing strategy sufficient for learning effective behaviors from multi-task offline data? In this section, we aim to answer this question via empirical analysis on a relatively simple domain, which will reveal interesting aspects of data sharing. We first describe the experimental setup and then discuss the results and possible explanations for the observed behavior. Using insights obtained from this analysis, we will then derive a simple and effective data sharing strategy in Section~\ref{sec:method}.

\textbf{Experimental analysis setup.} To assess the efficacy of data sharing, we experimentally analyze various multi-task RL scenarios created with the walker2d environment in Gym~\citep{brockman2016openai}. We construct different test scenarios on this environment that mimic practical situations, including settings where different amounts of  data of varied quality are available for different tasks~\citep{kalashnikov2021mt,xie2019improvisation,singh2020parrot}. In all these scenarios, the agent attempts three tasks: \texttt{run forward}, \texttt{run backward}, and \texttt{jump}, which we visualize in Figure~\ref{fig:env}. Following the problem statement in Section~\ref{sec:prelims}, these tasks share the same state-action space and transition dynamics, differing only in the reward function that the agent is trying to optimize. 
Different scenarios are generated with varying size offline datasets, each collected with policies that have different degrees of suboptimality. This might include, for each task, a single policy with mediocre or expert performance, or a mixture of policies given by the initial part of the replay buffer trained with online SAC~\citep{haarnoja2018soft}. We refer to these three types of offline datasets as medium, expert and medium-replay, respectively, following \citet{fu2020d4rl}.

We train a single-task policy $\pi_\mathrm{CQL}(\ba|\bs, i)$ with CQL~\citep{kumar2020conservative} as the base offline RL method, along with two forms of data-sharing, as shown in Table~\ref{tab:analysis}: no sharing of data across tasks (\textbf{No Sharing)}) and complete sharing of data with relabeling across all tasks (\textbf{Sharing All}). In addition, we also measure the divergence term in Equation~\ref{eqn:generic_multitask_offline_rl}, $D(\pi(\cdot|\cdot, i), \pi^\mathrm{eff}_\beta(\cdot|\cdot, i))$, for $\pi = \pi_\mathrm{CQL}(\ba|\bs, i)$, averaged across tasks by using the
Kullback-Liebler divergence. This value quantifies the average divergence between the single-task optimal policy and the relabeled behavior policy averaged across tasks.  

\begin{table*}[t]
  \centering
  \scriptsize
  \def\arraystretch{0.9}
  \setlength{\tabcolsep}{0.42em}
  
\begin{tabularx}{0.7\linewidth}{cc|cc|cc}
  \toprule
 \multicolumn{1}{c}{\multirow{1.5}[2]{*}{Dataset types / Tasks}} & \multicolumn{1}{c}{\multirow{1.5}[2]{*}{Dataset Size}}\vline &
 \multicolumn{2}{c}{Avg Return}\vline & \multicolumn{2}{c}{$D_\text{KL}(\pi, \pi_\beta)$}\\
 & \multicolumn{1}{c}{}\vline & \multicolumn{1}{c}{\textbf{No Sharing}}  & \multicolumn{1}{c}{\textbf{Sharing All}}\vline  & \multicolumn{1}{c}{\textbf{No Sharing}}  & \multicolumn{1}{c}{\textbf{Sharing All}} \\
\midrule
  medium-replay / \texttt{run forward} & 109900 & \textbf{998.9} & 966.2 & \textbf{3.70} & 10.39\\
  medium-replay / \texttt{run backward} & 109980 & \textbf{1298.6} & 1147.5 & \textbf{4.55} & 12.70\\
  medium-replay / \texttt{jump} & 109511 & \textbf{1603.1} & 1224.7 & \textbf{3.57} & 15.89\\
  \rowcolor{Gray}
  \textbf{average task performance} & N/A & \textbf{1300.2} & 1112.8 & \textbf{3.94}  & 12.99\\
  \midrule
  medium / \texttt{run forward} & 27646 & 297.4  & \textbf{848.7} &\textbf{6.53} & 11.78\\
  medium / \texttt{run backward} & 31298 & 207.5 & \textbf{600.4}& \textbf{4.44} & 10.13\\
  medium / \texttt{jump} & 100000 & 351.1 & \textbf{776.1}& \textbf{5.57} & 21.27\\
\rowcolor{Gray}
  \textbf{average task performance} & N/A & 285.3 & \textbf{747.7} & \textbf{5.51} & 14.39 \\
  \midrule
  medium-replay / \texttt{run forward} & 109900 & 590.1 & \textbf{701.4}& \textbf{1.49} & 7.76\\
  medium / \texttt{run backward} & 31298 & 614.7 & \textbf{756.7}& \textbf{1.91} & 12.2\\
  \rowcolor{yellow}
  expert / \texttt{jump} & 5000 & \textbf{1575.2} & 885.1 & \textbf{3.12} & 27.5\\
\rowcolor{Gray}
  \textbf{average task performance} & N/A & \textbf{926.6} & 781 & \textbf{2.17}  & 15.82 \\
    \bottomrule
    \end{tabularx}
    \vspace{-0.1cm}
         \caption{\footnotesize We analyze how sharing data across all tasks (\textbf{Sharing All}) compares to \textbf{No Sharing} in the multi-task walker2d environment with three tasks: run forward, run backward, and jump. We provide three scenarios with different styles of per-task offline datasets in the leftmost column. The second column shows the number of transitions in each dataset. We report the per-task average return, the KL divergence between the single-task optimal policy $\pi$ and the behavior policy $\behavior$ after the data sharing scheme, as well as averages across tasks. \textbf{Sharing All} generally helps training while increasing the KL divergence. However, on the row highlighted in yellow, \textbf{Sharing All} yields a particularly large KL divergence between the single-task $\pi$ and $\behavior$ and degrades the performance, suggesting sharing data for all tasks is brittle.
     \label{tab:analysis}
     \vspace{-0.6cm}
     }
\end{table*}

\textbf{Analysis of results in Table~\ref{tab:analysis}.} To begin, note that even na\"ively sharing data is  better than not sharing any data at all on \textbf{5/9} tasks considered
(compare the performance across \textbf{No Sharing} and \textbf{Sharing All} in Table~\ref{tab:analysis}).
However, a closer look at Table~\ref{tab:analysis} suggests that data-sharing can significantly degrade performance on certain tasks, especially in scenarios where the amount of data available for the original task is limited, and where the distribution of this data is narrow.
For example, when using expert data for jumping in conjunction with more than 25 times as much lower-quality (mediocre \& random) data for running forward and backward, we find that the agent performs poorly on the jumping task despite access to near-optimal jumping data.

\textbf{\emph{Why does na\"ive data sharing degrade performance on certain tasks despite near-optimal behavior for these tasks in the original task dataset?}} We argue that the primary reason that na\"{i}ve data sharing can actually hurt performance in such cases is because it exacerbates the distributional shift issues that afflict offline RL. Many offline RL methods combat distribution shift by implicitly or explicitly constraining the learned policy to stay close to the training data. Then, when the training data is changed by adding relabeled data from another task, the constraint causes the learned policy to change as well. When the added data is of low quality for that task, it will correspondingly lead to a lower quality learned policy for that task, unless the constraint is somehow modified.
This effect is evident from the higher divergence values between the learned policy without any data-sharing and the effective behavior policy for that task \emph{after} relabeling (e.g., expert+\texttt{jump}) in Table~\ref{tab:analysis}. Although these results are only for CQL, we expect that any offline RL method would, insofar as it combats distributional shift by staying close to the data, would exhibit a similar problem.

\textbf{To mathematically quantify} the effects of data-sharing in multi-task offline RL, we appeal to safe policy improvement bounds~\citep{laroche2019safe,kumar2020conservative,yu2021combo} and discuss cases where data-sharing between tasks $i$ and $j$ can degrade the amount of worst-case guaranteed improvement over the behavior policy. Prior work~\citep{kumar2020conservative} has shown that the generic offline RL algorithm in Equation~\ref{eqn:generic_offline_rl} enjoys the following guarantees of policy improvement on the actual MDP, beyond the behavior policy: 
\begin{align}
\label{eqn:spi}
    J(\pi^*) &\geq J(\pi_\beta) - \mathcal{O}(1/ (1 - \gamma)^2) \mathbb{E}_{\bs, \ba \sim d^{\pi}} \left[\sqrt{\frac{D(\pi(\cdot|\bs), \pi_\beta(\cdot|\bs))}{|\mathcal{D}(\bs)|}}\right] + \alpha/(1 - \gamma) D(\pi, \pi_\beta).
\end{align}
We will use Equation~\ref{eqn:spi} to understand the scenarios where data sharing can hurt. When data sharing modifies $\mathcal{D} = \mathcal{D}_i$ to $\mathcal{D} = \mathcal{D}^\mathrm{eff}_i$, which includes $\mathcal{D}_i$ as a subset, it effectively aims at reducing the magnitude of the second term (i.e., sampling error) by increasing the denominator. This can be highly effective if the state distribution of the learned policy $\pi^*$ and the dataset $\mathcal{D}$ overlap. However, an increase in the divergence $D(\pi(\cdot|\bs), \pi^\beta(\cdot|\bs))$ as a consequence of relabeling implies a potential increase in the sampling error, unless the increased value of $|\mathcal{D}^\mathrm{eff}(\bs)|$ compensates for this. Additionally, the bound
also depends on the quality of the behavior data added after relabeling: if the resulting behavior policy $\pi^\mathrm{eff}_\beta$ is more suboptimal compared to $\pi_\beta$, i.e., $J(\pi^\mathrm{eff}_\beta) < J(\pi_\beta)$, then the guaranteed amount of improvement also reduces.

\textbf{To conclude,} our analysis reveals that while data sharing is often helpful in multi-task offline RL, it can lead to substantially poor performance on certain tasks as a result of exacerbated distributional shift between the optimal policy and the effective behavior policy induced after sharing data.

\vspace{-5pt}
\section{\methodname: Reducing Distributional Shift in Multi-Task Data Sharing}
\label{sec:method}
\vspace{-5pt}
The analysis in Section~\ref{sec:analysis} shows that na\"ive data sharing may be highly sub-optimal in some cases, and although it often does improve over no data sharing at all in practice, it can also lead to exceedingly poor performance. Can we devise a conservative approach that shares data intelligently to not exacerbate distributional shift as a result of relabeling?

\subsection{A First Attempt at Designing a Data Sharing Strategy}
A straightforward data sharing strategy is to utilize a transition for training only if it reduces the distributional shift.
Formally, this means that for a given transition $(\bs, \ba, r_j(\bs, \ba), \bs') \in \mathcal{D}_j$ sampled from the dataset $\mathcal{D}_j$, such a scheme would prescribe using it for training task $i$ (i.e., $(\bs, \ba, r_i(\bs, \ba), \bs') \in \mathcal{D}^\mathrm{eff}_i$) only if: 
\begin{tcolorbox}[colback=blue!6!white,colframe=black,boxsep=0pt,top=-3pt,bottom=2pt]
\begin{equation}
\label{eqn:pessimistically_conservative}
    \text{\textbf{CDS (basic)}:}~~~~~~~\Delta^\pi(\bs, \ba) := D(\pi(\cdot|\cdot, i), \pi_\beta(\cdot|\cdot, i))(\bs) - D(\pi(\cdot|\cdot, i), \pi_\beta^{\text{eff}}(\cdot|\cdot, i))(\bs) \geq 0. 
\end{equation}
\end{tcolorbox}
The scheme presented in Equation~\ref{eqn:pessimistically_conservative} would guarantee that distributional shift (i.e., second term in Equation~\ref{eqn:generic_multitask_offline_rl}) is reduced.
Moreover, since sharing data can only increase the size of the dataset and not reduce it, this scheme is guaranteed to not increase the sampling error term in Equation~\ref{eqn:spi}. We refer to this scheme as the basic variant of conservative data sharing (\textbf{CDS (basic)}).

While this scheme can prevent the negative effects of increased distributional shift, this scheme is quite pessimistic. Even in our experiments, we find that this variant of CDS does not improve performance by a large margin.
Additionally, as observed in Table~\ref{tab:analysis} (medium-medium-medium data composition) and discussed in Section~\ref{sec:analysis}, data sharing can often be useful despite an increased distributional shift (note the higher values of $D_\mathrm{KL}(\pi, \pi_\beta)$ in Table~\ref{tab:analysis}) likely because it reduces sampling error and potentially utilizes data of higher quality for training. \textbf{CDS (basic)} described above does not take into account these factors. Formally, the effect of the first term in Equation~\ref{eqn:generic_multitask_offline_rl}, $J_{\mathcal{D}^\text{eff}}(\pi)$ (the policy return in the empirical MDP generated by the dataset) and a larger increase in $|\mathcal{D}^\mathrm{eff}(\bs)|$ at the cost of somewhat increased value of $D(\pi(\cdot|\bs), \pi_\beta(\cdot|\bs)$ are not taken into account. Thus we ask: can we instead design a more complete version of CDS that effectively balances the tradeoff by incorporating all the discussed factors (distributional shift, sampling error, data quality)?

\subsection{The Complete Version of Conservative Data Sharing (CDS)}
\label{sec:complete_cds}
\begin{wrapfigure}{r}{0.5\textwidth}
\centering
\vspace{-0.7cm}
\includegraphics[width=0.99\linewidth]{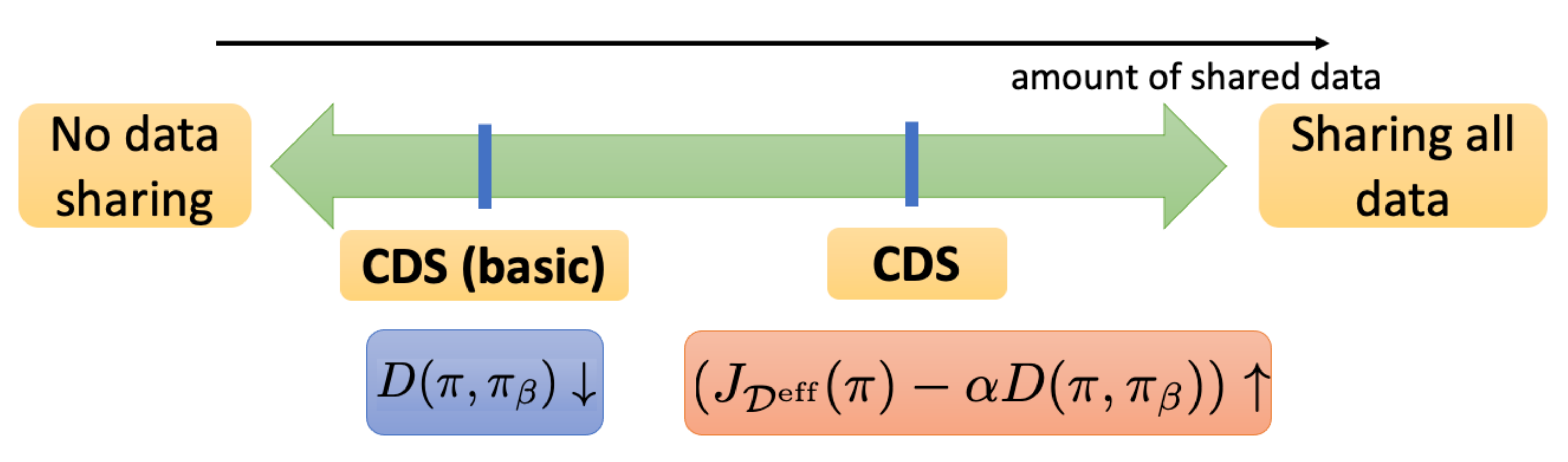}
\vspace{-0.6cm}
\caption{\label{fig:cds_variants_main} \footnotesize A schematic comparing \textbf{CDS} and \textbf{CDS (basic)} data sharing schemes relative to no sharing (left extreme) and full data sharing (right extreme). While p-CDS only shares data when distributional shift is strictly reduced, o-CDS is more optimistic and shares data when the objective in Equation~\ref{eqn:generic_multitask_offline_rl} is larger. Typically, we would expect that CDS shares more transitions than CDS (basic).}
\vspace{-0.6cm}
\end{wrapfigure}
Next, we present the complete version of our method. The complete version of CDS, which we will refer to as \textbf{CDS}, for notational brevity is derived from the following perspective: we note that a data sharing scheme can be viewed as altering the dataset $\mathcal{D}^\mathrm{eff}_i$, and hence the effective behavior policy, $\pi^\mathrm{eff}_\beta(\ba|\bs, i)$.
Thus, we can directly \emph{optimize} the objective in Equation~\ref{eqn:generic_multitask_offline_rl} with respect to $\pi^\mathrm{eff}_\beta$,
in addition to $\pi$, where $\pi^\mathrm{eff}_\beta$ belongs to the set of all possible effective behavior policies that can be obtained via any form of data sharing. Note that unlike CDS (basic), this approach would not rely on only indirectly controlling the objective in Equation~\ref{eqn:generic_multitask_offline_rl} by controlling distributional shift, but would aim to directly optimize the objective in Equation~\ref{eqn:generic_multitask_offline_rl}.  We formalize this optimization below in Equation~\ref{eqn:optimize_behavior}:
\begin{align}
\label{eqn:optimize_behavior}
    \arg \max_{\pi} \textcolor{red}{\max_{\pi^\mathrm{eff}_\beta \in \Pi_{\mathrm{relabel}}}}~~~ \left[J_{\mathcal{D}^\mathrm{eff}_i}(\pi) - \alpha D(\pi, \pi^\mathrm{eff}_\beta; i)\right],
\end{align}
where $\Pi_{\mathrm{relabel}}$ denotes the set of all possible behavior policies that can be obtained via relabeling. The next result characterizes safe policy improvement for Equation~\ref{eqn:optimize_behavior} and discusses how it leads to improvement over the behavior policy and also produces an effective practical method.
\begin{proposition}[Characterizing safe-policy improvement for CDS.] 
\label{prop:spi_thm}
Let $\pi^*(\ba|\bs)$ be the policy obtained by optimizing Equation~\ref{eqn:optimize_behavior}, and let $\pi_\beta(\ba|\bs)$ be the behavior policy for $\mathcal{D}_i$. Then, w.h.p. $\geq 1 - \delta$, $\pi^*$ is a $\zeta$-safe policy improvement over $\pi_\beta$, i.e., $J(\pi^*) \geq J(\pi_\beta) - \zeta$, where $\zeta$ is given by:
\begin{align*}
\small{
    \zeta = \mathcal{O}\left(\frac{1}{(1 - \gamma)^2}\right)  \mathbb{E}_{\bs \sim d^{\pi^*}_{\mathcal{D}^\mathrm{eff}_i}}\left[\sqrt{\frac{D_{\text{CQL}}(\pi^*, \pi_\beta^*)(\bs) + 1}{|\mathcal{D}^\mathrm{eff}_i(\bs)|}} \right]
    -  \left[\alpha D(\pi^*, \pi_\beta^*) + \underbrace{J(\pi^*_\beta) - J(\pi_\beta)}_{\text{(a)}} \right],}
\end{align*}
where $\mathcal{D}^\mathrm{eff}_i \sim d^{\pi_\beta^*}(\bs)$ and $\pi^*_\beta(\ba|\bs)$ denotes the policy $\pi \in \Pi_{\text{relabel}}$ that maximizes Equation~\ref{eqn:optimize_behavior}. 
\end{proposition}

A proof and analysis of this proposition is provided in Appendix~\ref{app:proofs},
where we note that the bound in Proposition~\ref{prop:spi_thm} is stronger than both no data sharing as well as na\"ive data sharing. We show in Appendix~\ref{app:proofs} that optimizing Equation~\ref{eqn:optimize_behavior} reduces the numerator $D_\mathrm{CQL}(\pi^*, \pi_\beta^*)$ term while also increasing $|\mathcal{D}^\mathrm{eff}_i(\bs)|$, thus reducing the amount of sampling error. In addition, Lemma~\ref{lemma:a_gt_0} shows that the improvement term $(a)$ is guaranteed to be positive if a large enough $\alpha$ is chosen in Equation~\ref{eqn:optimize_behavior}. Combining these, we find data sharing using Equation~\ref{eqn:optimize_behavior} improves over both complete data sharing (which may increase $D_\mathrm{CQL}(\pi, \pi_\beta)$) and no data sharing (which does not increase $|\mathcal{D}^\mathrm{eff}_i(\bs)|$). A schematic comparing the two variants of CDS and na\"ive and no data sharing schemes is shown in Figure~\ref{fig:cds_variants_main}.

{\textbf{Optimizing Equation~\ref{eqn:optimize_behavior} tractably.}} 
The next step is to effectively convert Equation~\ref{eqn:optimize_behavior} into a simple condition for data sharing in  multi-task offline RL. While directly solving Equation~\ref{eqn:optimize_behavior} is intractable in practice, since both the terms depend on $\pi^\mathrm{eff}_\beta(\ba|\bs)$ (since the first term $J_{\mathcal{D}^\mathrm{eff}}(\pi)$ depends on the empirical MDP induced by the effective behavior policy and the amount of sampling error), we need to instead  solve Equation~\ref{eqn:optimize_behavior} approximately. Fortunately, we can optimize a \textit{lower-bound approximation} to Equation~\ref{eqn:optimize_behavior} that uses the dataset state distribution for the policy update in Equation~\ref{eqn:optimize_behavior} similar to modern actor-critic methods~\citep{degris2012off,lillicrap2015continuous,fujimoto2018addressing,haarnoja2018soft,kumar2020conservative} which only introduces an additional $D(\pi, \pi_\beta)$ term in the objective. This objective is given by: $\mathbb{E}_{\bs \sim \mathcal{D}^{\mathrm{eff}}_i}[\mathbb{E}_\pi[Q(\bs, \ba, i)] - \alpha' D(\pi(\cdot|\bs,i), \pi_\beta^\mathrm{eff}(\cdot|\bs,i))]$, which is equal to the expected ``conservative Q-value'' $\hat{Q}^\pi(\bs, \ba, i)$ on dataset states, policy actions and task $i$. Optimizing this objective via a co-ordinate descent on $\pi$ and $\pi^\mathrm{eff}_\beta$ dictates that $\pi$ be updated using a standard update of maximizing the conservative Q-function, $\hat{Q}^\pi$ (equal to the difference of the Q-function and $D(\pi, \pi^{\mathrm{eff}}_\beta; i)$).
Moreover, $\pi^{\mathrm{eff}}_\beta$ should also be updated towards maximizing the same expectation, $\mathbb{E}_{\bs, \ba \sim \mathcal{D}^\mathrm{eff}_i}[\hat{Q}^\pi(\bs, \ba, i)] := \mathbb{E}_{\bs, \ba \sim \mathcal{D}^\mathrm{eff}_i}[Q(\bs, \ba, i)] - \alpha D(\pi, \pi_\beta^\mathrm{eff}; i)$. This implies that when updating the behavior policy during relabeling, we should prefer state-action pairs that maximize the conservative Q-function.

\vspace{-0.35cm}
\begin{algorithm}[H]
\begin{small}
  \caption{\methodname: Conservative Data Sharing}\label{alg:cds}
  \begin{algorithmic}[1]
    \REQUIRE Multi-task offline dataset $\cup_{i=1}^N \mathcal{D}_i$.
    \STATE Randomly initialize policy  $\pi_\theta(\ba|\bs, i)$.
    \FOR{$k=1, 2, 3, \cdots,$}
    \STATE Initialize $\mathcal{D}^\mathrm{eff} \leftarrow \{\}$
    \FOR{$i=1, \cdots, N$}
        \STATE $\mathcal{D}^\mathrm{eff}_i$ = $\mathcal{D}_i \cup \{(\bs_j, \ba_j, \bs'_j, r_i) \in \mathcal{D}_{j \rightarrow i}: \Delta^\pi(\bs, \ba) \geq 0\}$ using Eq.~\ref{eqn:method} (\methodname) or Eq.~\ref{eqn:p-cds} (\methodname (basic)).
    \ENDFOR
    \STATE Improve policy by solving eq.~\ref{eqn:generic_multitask_offline_rl} using samples from $\mathcal{D}^\mathrm{eff}$ to obtain $\policy_\theta^{k+1}$.
    \ENDFOR
  \end{algorithmic}
\end{small}
\end{algorithm}
\vspace{-0.5cm}

\textbf{{Deriving the data sharing strategy for CDS.}}
Utilizing the insights for optimizing Equation~\ref{eqn:optimize_behavior} tractably as discussed above, we now present the effective data sharing rule prescribed by CDS. For any given task $i$, we want relabeling to incorporate transitions with the highest conservative Q-value into the resulting dataset $\mathcal{D}^\mathrm{eff}_i$, as this will directly optimize the tractable lower bound on Equation~\ref{eqn:optimize_behavior}. While directly optimizing Equation~\ref{eqn:optimize_behavior} will enjoy benefits of reduced sampling error since $J_{\mathcal{D}^\mathrm{eff}_i}(\pi)$ also depends on sampling error, our tractable lower bound approximation does not enjoy this benefit. This is because optimizing the lower-bound only increases the frequency of a state in the dataset, $|\mathcal{D}^\mathrm{eff}_i(\bs)|$ by atmost 1. To encourage further reduction in sampling error, we modify CDS to instead share all transitions with a conservative Q-value more than the top $k^\text{th}$ quantile of the original dataset $\mathcal{D}_i$, where $k$ is a hyperparameter. This provably increases the objective value in Equation~\ref{eqn:optimize_behavior} still ensuring that term $(a) > 0$ in Proposition~\ref{prop:spi_thm}, while also reducing $|\mathcal{D}^{\mathrm{eff}}_i(\bs)|$ in the denominator. Thus, for a given transition $(\bs, \ba, \bs') \in \mathcal{D}_j$,  
\begin{tcolorbox}[colback=blue!6!white,colframe=black,boxsep=0pt,top=-5pt,bottom=5pt]
\begin{align}
   \text{\textbf{CDS}:~~~~~~~~} \!\!\!\!\!\!\!\!\!(\bs, \ba, r_i, \bs') \in \mathcal{D}^{\mathrm{eff}}_i  \text{~if~} {\Delta}^\pi(\bs, \ba)\! := \hat{Q}^\pi(\bs, \ba, i) - P_{k\%}\!\left\{\!\hat{Q}^\pi(\bs', \ba', i)\!\!: \bs', \ba' \sim \mathcal{D}_i\!\right\} \geq 0,
\label{eqn:method}
\end{align}
\end{tcolorbox}
where $\hat{Q}^\pi$ denotes the learned conservative Q-function estimate. If the condition in Equation~\ref{eqn:method} holds for the given $(\bs, \ba)$, then the corresponding relabeled transition, $(\bs, \ba, r_{i}(\bs, \ba), \bs')$ is added to $\mathcal{D}^\mathrm{eff}_i$.

\subsection{{Practical implementation of \methodname}} 
The pseudocode of CDS is summarized in Algorithm~\ref{alg:cds}. The complete variant of CDS can be directly implemented using the rule in Equation~\ref{eqn:method} with conservative Q-value estimates obtained via any offline RL method that constrains the learned policy to the behavior policy. For implementing \methodname\ (basic), we reparameterize the divergence $D$ in Equation~\ref{eqn:pessimistically_conservative} to use the learned conservative Q-values. This is especially useful for our implementation since we utilize CQL as the base offline RL method, and hence we do not have access to an explicit divergence. In this case, $\Delta^\pi(\bs, \ba)$ can be redefined as, $\Delta^\pi(\bs, \ba) :=$
\begin{align}
    \mathbb{E}_{\bs' \sim \mathcal{D}^i}\left[\mathbb{E}_{\ba' \sim \pi}[\hat{Q}(\bs', \ba', i)] - \mathbb{E}_{\ba'' \sim \mathcal{D}_i}[\hat{Q}(\bs', \ba'', i)]\right] - \left(\mathbb{E}_{\ba' \sim \pi}[\hat{Q}(\bs, \ba', i)] - Q(\bs, \ba, i)\right),
\label{eqn:p-cds}
\end{align}
Equation~\ref{eqn:p-cds} can be viewed as the difference between the CQL~\citep{kumar2020conservative} regularization term on a given $(\bs, \ba)$ and the original dataset for task $i$, $\mathcal{D}_i$. This CQL regularization term is equal to the divergence between the learned policy $\pi(\cdot|\bs)$ and the behavior policy $\pi_\beta(\cdot|\bs)$, therefore Equation~\ref{eqn:p-cds} practically computes Equation~\ref{eqn:pessimistically_conservative}. 

Finally, both variants of \methodname\ train a policy, 
$\pi(\ba|\bs; i)$, either conditioned on the task $i$ (i.e., with weight sharing) or a separate $\pi(\ba|\bs)$ policy for each task with no weight sharing, using the resulting relabeled dataset, $\mathcal{D}^\mathrm{eff}_i$.
For practical implementation details of \methodname\, see Appendix~\ref{app:details}.

\section{Experimental Evaluation}
\label{sec:exp}

We conduct experiments to answer six main questions: \textbf{(1)} can \methodname\ prevent performance degradation when sharing data as observed in Section~\ref{sec:analysis}?, \textbf{(2)} how does \methodname\ compare to vanilla multi-task offline RL methods and prior data sharing methods?
\textbf{(3)} can \methodname\ handle sparse reward settings, where data sharing is particularly important due to scarce supervision signal? \textbf{(4)} can \methodname\ handle goal-conditioned offline RL settings where the offline dataset is undirected and highly suboptimal? \textbf{(5)} Can \methodname\ scale to complex visual observations? \arxiv{\textbf{(6)} Can \methodname\ be combined with any offline RL algorithms? Besides these questions, we visualize CDS weights for better interpretation of the data sharing scheme learned by CDS in Figure~\ref{fig:antmaze_vis} in Appendix~\ref{app:cds_vis}}.

\begin{wrapfigure}{r}{0.65\textwidth}
    \vspace{-0.65cm}
    \centering
    \includegraphics[width=0.61\textwidth]{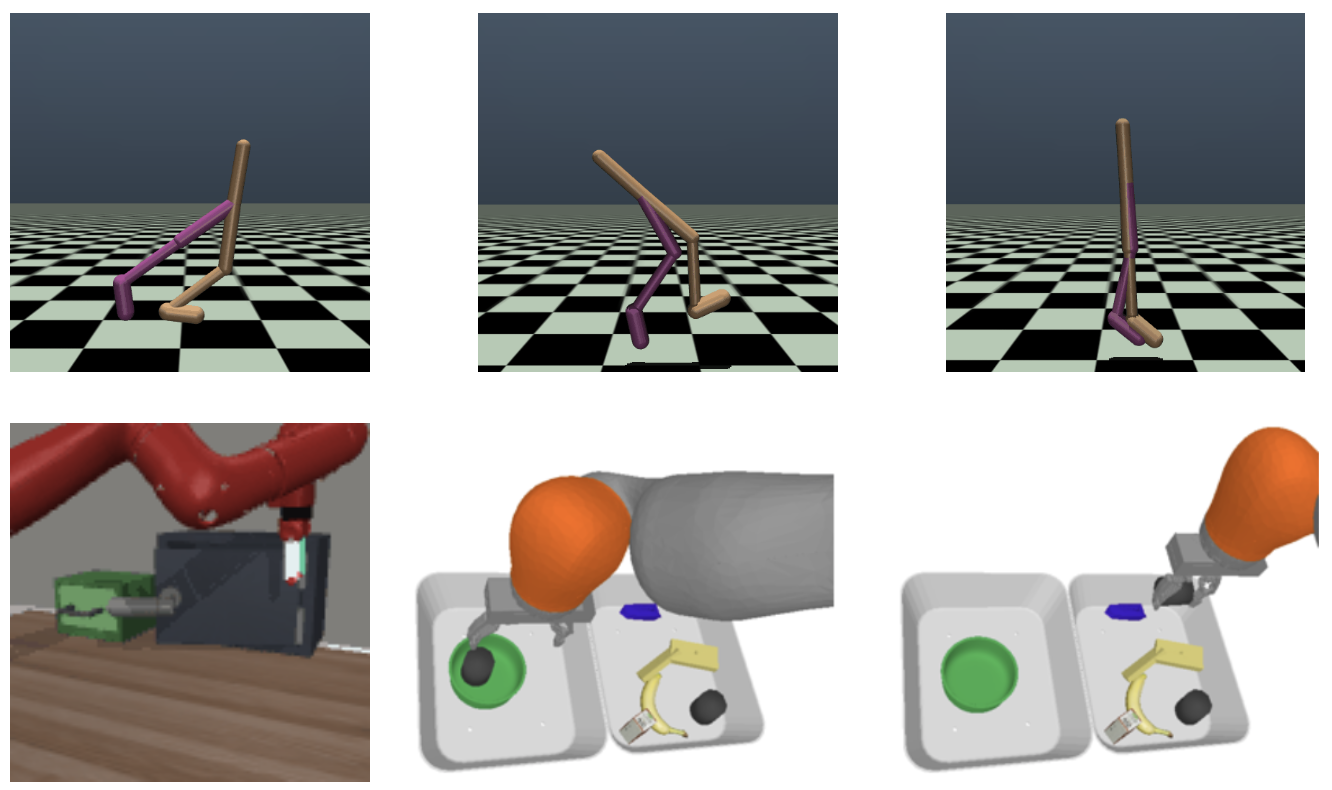}
    \vspace{-0.32cm}
    \caption{\footnotesize  Environments (from left to right): walker2d {run forward}, walker2d {run backward}, walker2d {jump},  Meta-World {door open/close} and {drawer open/close} and vision-based pick-place tasks in \citep{kalashnikov2021mt}.}
    \label{fig:env}
    \vspace{-0.4cm}
\end{wrapfigure}
\textbf{Comparisons.} To answer these questions, we consider the following prior methods. On tasks with low dimensional state spaces, we compare with the online multi-task relabeling approach \textbf{HIPI}~\citep{eysenbach2020rewriting}, which uses inverse RL to infer for which tasks the datapoints are optimal and in practice routes a transition to task with the highest Q-value. We adapt HIPI to the offline setting by applying its data routing strategy to a conservative offline RL algorithm.
We also compare to na\"ively sharing data across all tasks (denoted as \textbf{Sharing All}) and vanilla multi-task offline RL method without any data sharing (denoted as \textbf{No Sharing}). On image-based domains, we compare \methodname\ to the data sharing strategy based on human-defined skills~\citep{kalashnikov2021mt} (denoted as \textbf{Skill}), which manually groups tasks into different skills (e.g. skill ``pick'' and skill ``place'') and only routes an episode to target tasks that belongs to the same skill.
In these domains, we also compare to \textbf{HIPI}, \textbf{Sharing All} and \textbf{No Sharing}. \arxiv{Beyond these multi-task RL approaches with data sharing, to assess the importance of data sharing in offline RL, we perform an additional comparison to other alternatives to data sharing in multi-task offline RL settings. One traditionally considered approach is to use data from other tasks for some form of ``pre-training'' before learning to solve the actual task. We instantiate this idea by considering a method from \citet{yang2021representation} that conducts contrastive representation learning on the multi-task datasets to extract shared representation between tasks and then runs multi-task RL on the learned representations. We discuss this comparison in detail in Table~\ref{tbl:pretrain_comparison} in Appendix~\ref{app:pretrain_comparison}.} 
To answer question (6), we use CQL~\citep{kumar2020conservative} (a Q-function regularization method) and BRAC~\citep{wu2019behavior} (a policy-constraint method) as the base offline RL algorithms for all methods. \arxiv{We discuss evaluations of CDS with CQL in the main text and include the results of CDS with BRAC in Table~\ref{tbl:brac_comparison} in Appendix~\ref{app:brac_results}.} 
For more details on setup and hyperparameters, see Appendix~\ref{app:details}.

\textbf{Multi-task environments.} We consider a number of multi-task reinforcement learning problems on environments visualized in Figure~\ref{fig:env}. 
\arxiv{To answer questions (1) and (2), we consider the walker2d locomotion environment from OpenAI Gym~\citep{brockman2016openai} with dense rewards. We use three tasks, \texttt{run forward}, \texttt{run backward} and \texttt{jump}, as proposed in prior offline RL work~\citep{yu2020mopo}.}
To answer question (3), we also evaluate on robotic manipulation domains using environments from the Meta-World benchmark~\citep{yu2020metaworld}. We consider four tasks: \texttt{door open}, \texttt{door close}, \texttt{drawer open} and \texttt{drawer close}. Meaningful data sharing requires a consistent state representation across tasks, so we put both the door and the drawer on the same table, as shown in Figure~\ref{fig:env}. Each task has a sparse reward of 1 when the success condition is met and 0 otherwise. To answer question (4), we consider maze navigation tasks where the temporal ``stitching'' ability of an offline RL algorithm is crucial to obtain good performance. We create goal reaching tasks using the ant robot in the medium and hard mazes from D4RL~\citep{fu2020d4rl}. The set of goals is a fixed discrete set of size 7 and 3 for large and medium mazes, respectively. Following \citet{fu2020d4rl}, a reward of +1 is given and the episode terminates if the state is within a threshold radius of the goal. Finally, to explore how \methodname\ scales to image-based manipulation tasks (question (5)), we utilize a simulation environment similar to the real-world setup presented in~\citep{kalashnikov2021mt}. This environment, which was utilized by \citet{kalashnikov2021mt} as a representative and realistic simulation of a real-world robotic manipulation problem, consists of 10 image-based manipulation tasks that involve different combinations of picking specific objects (banana, bottle, sausage, milk box, food box, can and carrot) and placing them in one of the three fixtures (bowl, plate and divider plate) (see example task images in Fig.~\ref{fig:env}).
We pick these tasks due to their similarity to the real-world setup introduced by \citet{kalashnikov2021mt}, which utilized a skill-based data-sharing heuristic strategy (\textbf{Skill}) for data-sharing that significantly outperformed simple data-sharing alternatives, which we use as a point of comparison.
More environment details are in the appendix. We report the average return for locomotion tasks and success rate for AntMaze and both manipluation environments, averaged over 6 and 3 random seeds for environments with low-dimensional inputs and image inputs respectively.

\begin{table}[h]
\vspace{0.1cm}
  \centering
  \scriptsize
  \def\arraystretch{0.9}
  \setlength{\tabcolsep}{0.42em}
  \vspace{-0.4cm}
\begin{tabularx}{0.75\linewidth}{cc|cccc}
  \toprule
 \multicolumn{1}{c}{\multirow{1.5}[2]{*}{Environment}} & \multicolumn{1}{c}{\multirow{1.5}[2]{*}{Dataset types / Tasks}}\vline &
 \multicolumn{3}{c}{$D_\text{KL}(\pi, \pi_\beta)$}\\
& \multicolumn{1}{c}{} \vline& \multicolumn{1}{c}{\textbf{No Sharing}}  & \multicolumn{1}{c}{\textbf{Sharing All}} & \multicolumn{1}{c}{\textbf{CDS (basic) (ours)}}  & \multicolumn{1}{c}{\textbf{CDS (ours)}}\\
\midrule
  &medium-replay / run forward & \textbf{1.49} & 7.76 & 14.31 & \textbf{1.49}\\
  walker2d& medium / run backward &  \textbf{1.91} & 12.2 & 8.26 & 6.09\\
  & \cellcolor{yellow} expert / jump & \cellcolor{yellow} 3.12 & \cellcolor{yellow} 27.5 & \cellcolor{yellow} 13.25  & \cellcolor{yellow} \textbf{2.91}\\
    \bottomrule
    \end{tabularx}
    \vspace{0.1cm}
         \caption{\footnotesize Measuring $D_\text{KL}(\pi, \pi_\beta)$ on the walker2d environment.  \textbf{Sharing All} degrades the performance on task \text{jump} with limited expert data as discussed in Table~\ref{tab:analysis} (M-R refers to medium-replay, M refers to medium, and E refers to expert). \methodname\ manages to obtain a $\behavior$ after data sharing that is closer to the single-task optimal policy in terms of the KL divergence compared to \textbf{No Sharing} and \textbf{Sharing All} on task \texttt{jump} (highlighted in yellow). Since \methodname\ also achieves better performance, this analysis suggests that reducing distribution shift is important for effective offline data sharing.
     \label{tab:analysis_cds}
     \vspace{-0.5cm}
     }
\end{table}

\textbf{Multi-task datasets.}  Following the analysis in Section~\ref{sec:analysis}, we intentionally construct datasets with a variety of heterogeneous behavior policies to test if \methodname\ can provide effective data sharing to improve performance while avoiding harmful data sharing that exacerbates distributional shift. For the locomotion domain, we use a large, diverse dataset (medium-replay) for \texttt{run forward}, a medium-sized dataset for \texttt{run backward}, and an expert dataset with limited data for \texttt{run jump}. For Meta-World, we consider medium-replay datasets with 152K transitions for task \texttt{door open} and \texttt{drawer close} and expert datasets with only 2K transitions for task \texttt{door close} and \texttt{drawer open}. For AntMaze, we modify the D4RL datasets for antmaze-*-play environments to construct two kinds of multi-task datasets: an ``undirected'' dataset, where data is equally divided between different tasks and the rewards are correspondingly relabeled, and a ``directed'' dataset, where a trajectory is associated with the goal closest to the final state of the trajectory. This means that the per-task data in the undirected setting may not be relevant to reaching the goal of interest. Thus, data-sharing is crucial for good performance: methods that do not effectively perform data sharing and train on largely task-irrelevant data are expected to perform worse. Finally, for image-based manipulation tasks, we collect datasets for all the tasks individually by running online RL~\cite{kalashnikov2018scalable} until the task reaches medium-level performance (40\% for picking tasks and 80\% placing tasks). At that point, we merge the entire replay buffers from different tasks creating a final dataset of 100K RL episodes with 25 transitions for each episode.

\begin{table*}[t!]
\centering
\vspace*{0.1cm}
\scriptsize
\resizebox{\textwidth}{!}{\begin{tabular}{l|l|r|r|r|r|r}
\toprule
\textbf{Environment} & \textbf{Tasks / Dataset type} & \textbf{\methodname\ (ours)} & \textbf{\methodname\ (basic)} & \textbf{HIPI}~\cite{eysenbach2020rewriting}& \textbf{Sharing All} & \textbf{No Sharing}\\ \midrule
& run forward / medium-replay & \textbf{1057.9}$\pm$121.6 & 968.6$\pm$188.6 & 695.5$\pm$61.9 & 701.4$\pm$47.0 & 590.1$\pm$48.6\\
walker2d & run backward / medium & 564.8$\pm$47.7 & 594.5$\pm$22.7 & 626.0$\pm$48.0& \textbf{756.7}$\pm$76.7& 614.7$\pm$87.3\\
& jump / expert & 1418.2$\pm$138.4 & 1501.8$\pm$115.1  & \textbf{1603.7}$\pm$146.8 & 885.1$\pm$152.9 & 1575.2$\pm$70.9\\
& \CC \textbf{average} & \CC 1013.6$\pm$71.5 &\CC \textbf{1021.6}$\pm$76.9 & \CC 975.1$\pm$45.1 & \CC 781.0$\pm$100.8 & \CC 926.6$\pm$37.7\\\midrule
& door open / medium-replay & \textbf{58.4\%}$\pm$9.3\% & 30.1\%$\pm$16.6\% & 26.5\%$\pm$20.5\% & 34.3\%$\pm$17.9\% & 14.5\%$\pm$12.7\\
& door close / expert & \textbf{65.3\%}$\pm$27.7\% & 41.5\%$\pm$28.2\% & 1.3\%$\pm$5.3\% & 48.3\%$\pm$27.3\% & 4.0\%$\pm$6.1\% \\
Meta-World~\citep{yu2020metaworld}& drawer open / expert & \textbf{57.9\%}$\pm$16.2\% & 39.4\%$\pm$16.9\% & 41.2\%$\pm$24.9\% & 55.1\%$\pm$9.4\% & 16.0\%$\pm$17.5\%\\
& drawer close / medium-replay & 98.8\%$\pm$0.7\% & 86.3\%$\pm$0.9\% & 62.2\%$\pm$33.4\% & \textbf{100.0\%}$\pm$0\% & 99.0\%$\pm$0.7\%\\
& \CC \textbf{average} & \CC \textbf{70.1\%}$\pm$8.1\% & \CC 49.3\%$\pm$16.0\% & \CC 32.8\%$\pm$18.7\% & \CC 59.4\%$\pm$5.7\% & \CC 33.4\%$\pm$8.3\%\\
\midrule
& large maze (7 tasks) / undirected & \textbf{22.8}\% $\pm$ 4.5\% & 10.0\% $\pm$ 5.9\% & 1.3\% $\pm$ 2.3\%  & 16.7\% $\pm$ 7.0\% & 13.3\% $\pm$ 8.6\% \\
AntMaze~\citep{fu2020d4rl}  & large maze (7 tasks) / directed &  \textbf{24.6\%} $\pm$ 4.7\% & 0.0\% $\pm$ 0.0\% & 11.8\% $\pm$ 5.4\% & 20.6\% $\pm$ 4.4\% & 19.2\% $\pm$ 8.0\% \\
& medium maze (3 tasks) / undirected &  \textbf{36.7\%} $\pm$ 6.2\% & 0.0\% $\pm$ 0.0\% & 8.6\% $\pm$ 3.2\% & 22.9\% $\pm$ 3.6\% & 21.6\% $\pm$ 7.1\% \\
& medium maze (3 tasks) / directed &  \textbf{18.5}\% $\pm$ 6.0\% & 0.0\% $\pm$ 0.0\% & 8.3\% $\pm$ 9.1\% & 12.4\% $\pm$ 5.4\% & \textbf{17.0\%} $\pm$ 3.2\% \\
\bottomrule
\end{tabular}}
\vspace{-0.2cm}
\caption{\footnotesize Results for multi-task locomotion (walker2d), robotic manipulation (Meta-World) and navigation environments (AntMaze) with low-dimensional state inputs.
 \arxiv{Numbers are averaged across 6 seeds, $\pm$ the 95$\%$-confidence interval.} We include per-task performance for walker2d and Meta-World domains and the overall performance averaged across tasks (highlighted in gray) for all three domains. We bold the highest score across all methods. \methodname\ performs achieves the best or comparable performance on all of these multi-task environments.
}
\label{tbl:gym}
\normalsize
\vspace{-0.3cm}
\end{table*}

\begin{table*}[t!]
\small{
\centering
\vspace*{0.1cm}
\resizebox{\textwidth}{!}{\begin{tabular}{l|r|r|r|r|r}
\toprule
\textbf{Task Name} & \textbf{\methodname\ (ours)}& \textbf{HIPI}~\cite{eysenbach2020rewriting} & \textbf{Skill~\cite{kalashnikov2021mt}} & \textbf{Sharing All} & \textbf{No Sharing}\\ \midrule
\texttt{lift-banana} & \textbf{53.1\%}$\pm$3.2\% & 48.3\%$\pm$6.0\%  & 32.1\%$\pm$9.5\% & 41.8\%$\pm$4.2\% & 20.0\%$\pm$6.0\%\\
\texttt{lift-bottle} & \textbf{74.0\%}$\pm$6.3\% & 64.4\%$\pm$7.7\%  & 55.9\%$\pm$9.6\% & 60.1\%$\pm$10.2\% & 49.7\%$\pm$8.7\%\\
\texttt{lift-sausage} & \textbf{71.8\%}$\pm$3.9\%  & 71.0\%$\pm$7.7\%  & 68.8\%$\pm$9.3\% & 70.0\%$\pm$7.0\% & 60.9\%$\pm$6.6\%\\
\texttt{lift-milk}& \textbf{83.4\%}$\pm$5.2\% & 79.0\%$\pm$3.9\% & 68.2\%$\pm$3.5\% & 72.5\%$\pm$5.3\% & 68.4\%$\pm$6.1\%\\

\texttt{lift-food} & 61.4\%$\pm$9.5\% & \textbf{62.6\%}$\pm$6.3\% & 41.5\%$\pm$12.1\% & 58.5\%$\pm$7.0\% & 39.1\%$\pm$7.0\%\\
\texttt{lift-can} & 65.5\%$\pm$6.9\% & \textbf{67.8\%}$\pm$6.8\%  & 50.8\%$\pm$12.5\% & 57.7\%$\pm$7.2\% & 49.1\%$\pm$9.8\%\\
\texttt{lift-carrot} & \textbf{83.8\%}$\pm$3.5\% & 78.8\%$\pm$6.9\% & 66.0\%$\pm$7.0\%& 75.2\%$\pm$7.6\%& 69.4\%$\pm$7.6\%\\
\texttt{place-bowl} & \textbf{81.0\%}$\pm$8.1\%  & 77.2\%$\pm$8.9\% & 80.8\%$\pm$6.9\% & 70.8\%$\pm$7.8\% & 80.3\%$\pm$8.6\%\\
\texttt{place-plate} & 85.8\%$\pm$6.6\%  & 83.6\%$\pm$7.9\% & 78.4\%$\pm$9.6\% & 78.7\%$\pm$7.6\% & \textbf{86.1}\%$\pm$7.7\%\\
\texttt{place-divider-plate} & \textbf{87.8\%}$\pm$7.6\%  & 78.0\%$\pm$10.5\% & 80.8\%$\pm$5.3\% & 79.2\%$\pm$6.3\% & 85.0\%$\pm$5.9\%\\
\CC \textbf{average} & \CC \textbf{74.8\%}$\pm$6.4\%  & \CC 71.1\%$\pm$7.5\% & \CC 62.3\%$\pm$8.9\% & \CC 66.4\%$\pm$7.2\% & \CC 60.8\%$\pm$7.5\%\\
\bottomrule
\end{tabular}}
\vspace{-0.1cm}
\caption{\footnotesize Results for multi-task vision-based robotic manipulation domains in \citep{kalashnikov2021mt}. \arxiv{Numbers are averaged across 3 seeds, $\pm$ the 95$\%$ confidence interval.} We consider 7 tasks denoted as \texttt{lift-object} where the goal of each task is to lift a different object and 3 tasks denoted as \texttt{place-fixture} where the goal of each task is to place a lifted object onto different fixtures. \methodname\ outperforms both a skill-based data sharing strategy~\citep{kalashnikov2021mt} (\textbf{Skill}) and other data sharing methods on the average task success rate (highlighted in gray) and 7 out of 10 per-task success rate.
}
\label{tbl:mtopt}
}
\vspace{-0.6cm}
\end{table*}

\textbf{Results on domains with low-dimensional states.} We present the results on all non-vision environments in Table~\ref{tbl:gym}.
\methodname\ achieves the 
best average performance across all environments except that on walker2d, it achieves the second best performance, obtaining slightly worse the other variant \methodname\ (basic). On the locomotion domain, we observe the most 
significant improvement on task \texttt{jump} on all three environments. We interpret this as strength of conservative data sharing, which mitigates the distribution shift that can be introduced by routing large amount of other task data to the task with limited data and narrow distribution. We also validate this by measuring the $D_\text{KL}(\pi, \pi_\beta)$ in Table~\ref{tab:analysis_cds} where $\pi_\beta$ is the behavior policy after we perform \methodname\
to share data. As shown in Table~\ref{tab:analysis_cds}, \methodname\ achieves lower KL divergence
 between the single-task optimal policy and the behavior policy after data sharing on task \texttt{jump} with limited expert data, whereas \textbf{Sharing All} results in much higher KL divergence compared to \textbf{No Sharing} as discussed in Section~\ref{sec:analysis} and Table~\ref{tab:analysis}. Hence, \methodname\ is able to mitigate distribution shift when sharing data and result in performance boost.

On the Meta-World tasks, we find that the agent without data sharing completely fails to solve most of the tasks due to the low quality of the medium replay datasets and the insufficient data for the expert datasets. \textbf{Sharing All} improves performance since in the sparse reward settings, data sharing can introduce more supervision signal and help training. \methodname\ further improves over \textbf{Sharing All}, suggesting that \methodname\ can not only prevent harmful data sharing, but also lead to more effective multi-task learning compared to \textbf{Sharing All} in scenarios where data sharing is imperative. It's worth noting that \methodname\ (basic) performs worse than \methodname\ and \textbf{Sharing All}, indicating that relabeling data that only mitigates distributional shift is too pessimistic and might not be sufficient to discover the shared structure across tasks.

In the AntMaze tasks, we observe that \methodname\ performs better than \textbf{Sharing All} and significantly outperforms HIPI in all four settings. Perhaps surprisingly, \textbf{No Sharing} is a strong baseline, however, is outperformed by \methodname\ with the harder undirected data. Moreover, \methodname\ performs on-par or better in the undirected setting compared to the directed setting, indicating the effectiveness of \methodname\ in routing data in challenging settings.

\textbf{Results on image-based robotic manipulation domains.}  %
Here, we compare \methodname\ to the hand-designed \textbf{Skill} sharing strategy, in addition to the other methods. \arxiv{Given that \methodname\ achieves significantly better performance than \methodname\ (basic) on low-dimensional robotic manipulation tasks in Meta-World, we only evaluate \methodname\ in the vision-based robotic manipulation domains.}  Since \methodname\ is applicable to any offline multi-task RL algorithm, we employ it as a separate data-sharing strategy in \citep{kalashnikov2021mt} while keeping the model architecture and all the other hyperparameters constant, which allows us to carefully evaluate the influence of data sharing in isolation. The results are reported in Table~\ref{tbl:mtopt}. \methodname\ outperforms both \textbf{Skill} and other approaches, indicating that \methodname\ is able to scale to high-dimensional observation inputs and can effectively remove the need for manual curation of data sharing strategies.

\vspace{-5pt}
\section{Conclusion}
\vspace{-5pt}
\label{sec:conclusion}
In this paper, we study the multi-task offline RL setting, focusing on the problem of sharing offline data across tasks for better multi-task learning. Through empirical analysis, we identify that na\"{i}vely sharing data across tasks generally helps learning but can significantly hurt performance in scenarios where excessive distribution shift is introduced. To address this challenge, we present conservative data sharing (CDS), which relabels data to a task when the conservative Q-value of the given transition is better than the expected conservative Q-value of the target task. On multitask locomotion, manipulation, navigation, and vision-based manipulation domains, CDS consistently outperforms or achieves comparable performance to existing data sharing approaches. While CDS attains superior results, it is not able to handle data sharing in settings where dynamics vary across tasks and requires functional forms of rewards. We leave these as future work. 

\section*{Acknowledgements}
We thank Kanishka Rao, Xinyang Geng, Avi Singh, other members of RAIL at UC Berkeley, IRIS at Stanford and Robotics at Google and anonymous reviewers for valuable and constructive feedback on an early version of this manuscript. This research was funded in part by Google, ONR grant N00014-20-1-2675, Intel Corporation and the DARPA Assured Autonomy Program. CF is a CIFAR Fellow in the Learning in Machines and Brains program.

\bibliography{reference}
\bibliographystyle{plainnat}

\newpage
\appendix
\onecolumn
\part*{Appendices}

\section{Analysis of \methodname}
\label{app:proofs}

In this section, we will analyze the key idea behind our method \methodname\ (Section~\ref{sec:method}) and show that the abstract version of our method (Equation~\ref{eqn:optimize_behavior}) provides better policy improvement guarantees than na\"ive data sharing and that the practical version of our method (Equation~\ref{eqn:method}) approximates Equation~\ref{eqn:optimize_behavior} resulting in an effective practical algorithm.

\subsection{Analysis of the Algorithm in Equation~\ref{eqn:optimize_behavior}}
We begin with analyzing Equation~\ref{eqn:optimize_behavior},
which is used to derive the practical variant of our method, \methodname. We build on the analysis of safe-policy improvement guarantees of conventional offline RL algorithms~\citep{laroche2019safe,kumar2020conservative} and show that data sharing using \methodname\ attains better guarantees in the worst case. To begin the analysis, we introduce some notation and prior results that we will directly compare to.

\textbf{Notation and prior results.}
Let $\pi_\beta(\ba|\bs)$ denote the behavior policy for task $i$ (note that index $i$ was dropped from $\pi_\beta(\ba|\bs; i)$ for brevity). The dataset, $\mathcal{D}_i$ 
is generated from the marginal state-action distribution of $\pi_\beta$, i.e., $\mathcal{D} \sim d^{\pi_\beta}(\bs) \pi_\beta(\ba|\bs)$. We define $d^{\pi}_{\mathcal{D}}$ as the state marginal distribution introduced by the dataset $\mathcal{D}$ under $\pi$. Let $D_\text{CQL}(p, q)$ denote the following distance between two distributions $p(\bx)$ and $q(\bx)$ with equal support $\mathcal{X}$:
\begin{equation*}
    D_\text{CQL}(p, q) := \sum_{\bx \in \mathcal{X}} p(\bx) \left(\frac{p(\bx)}{q(\bx)} - 1 \right).
\end{equation*}
Unless otherwise mentioned, we will drop the subscript ``CQL'' from $D_\text{CQL}$ and use $D$ and $D_\text{CQL}$ interchangeably. Prior works~\citep{kumar2020conservative} have shown that the optimal policy $\pi^*_{i}$ that optimizes Equation~\ref{eqn:generic_offline_rl} attains a high probability safe-policy improvement guarantee, i.e., $J(\pi^*_i) \geq J(\pi_\beta) - \zeta_i$, where $\zeta_i$ is:
\begin{equation}
    \label{eqn:single_task_guarantee}
    \zeta_i =  \mathcal{O}\left(\frac{1}{(1 - \gamma)^2}\right) \mathbb{E}_{\bs \sim d^{\pi^*_i}_{\mathcal{D}_i}}\left[\sqrt{\frac{D_{\text{CQL}}(\pi^*_i, \pi_\beta)(\bs) + 1}{|\mathcal{D}_i(\bs)|}} \right] + \alpha D(\pi^*_i, \pi_\beta).
\end{equation}
The first term in Equation~\ref{eqn:single_task_guarantee} corresponds to the decrease in performance due to sampling error and this term is high when the single-task optimal policy $\pi^*_i$ visits rarely observed states in the dataset $\mathcal{D}_i$ and/or when the divergence from the behavior policy $\pi_\beta$ is higher under the states visited by the single-task policy $\bs \sim d^{\pi^*_i}_{\mathcal{D}_i}$. 

Let $J_\mathcal{D}(\pi)$ denote the return of a policy $\pi$ in the empirical MDP induced by the transitions in the dataset $\mathcal{D}$. Further, let us assume that optimizing Equation~\ref{eqn:optimize_behavior} gives us the following policies:
\begin{equation}
    \label{eqn:eqn4_detailed}
    \pi^*(\ba|\bs), \pi^*_\beta(\ba|\bs) := \arg \max_{\pi, \pi_\beta \in \Pi_{\text{relabel}}}~~ \underbrace{J_{\mathcal{D}^\mathrm{eff}_i}(\pi) - \alpha D(\pi, \pi_\beta)}_{:= f(\pi, \pi_\beta; \mathcal{D}^\mathrm{eff}_i)},
\end{equation}
where the optimized behavior policy $\pi^*_\beta$ is constrained to lie in a set of all policies that can be obtained via relabeling, $\Pi_\text{relabel}$, and the dataset, $\mathcal{D}^\mathrm{eff}_i$ is sampled according to the state-action marginal distribution of $\pi^*_\beta$, i.e., $\mathcal{D}^\mathrm{eff}_i \sim d^{\pi^*_\beta}(\bs, \ba)$. Additionally, for convenience, define, $f(\pi_1, \pi_2; \mathcal{D}) := J_\mathcal{D}(\pi_1) - \alpha D(\pi_1, \pi_2)$ for any two policies $\pi_1$ and $\pi_2$, and a given dataset $\mathcal{D}$. 

We now show the following result for \methodname:

\begin{proposition}[Proposition~\ref{prop:spi_thm} restated] 
\label{prop:spi}
Let $\pi^*(\ba|\bs)$ be the policy obtained by optimizing Equation~\ref{eqn:optimize_behavior}, and let $\pi_\beta(\ba|\bs)$ be the behavior policy for $\mathcal{D}_i$. Then, w.h.p. $\geq 1 - \delta$, $\pi^*$ is a $\zeta$-safe policy improvement over $\pi_\beta$, i.e., $J(\pi^*) \geq J(\pi_\beta) - \zeta$, where $\zeta$ is given by:
\begin{equation*}
    \zeta = \mathcal{O}\left(\frac{1}{(1 - \gamma)^2}\right) \mathbb{E}_{\bs \sim d^{\pi^*}_{\mathcal{D}^\mathrm{eff}_i}}\left[\sqrt{\frac{D_{\text{CQL}}(\pi^*, \pi_\beta^*)(\bs) + 1}{|\mathcal{D}^\mathrm{eff}_i(\bs)|}} \right] -  \left[\alpha D(\pi^*, \pi_\beta^*) + \underbrace{J(\pi^*_\beta) - J(\pi_\beta)}_{\text{(a)}} \right],
\end{equation*}
where $\mathcal{D}^\mathrm{eff}_i \sim d^{\pi_\beta^*}(\bs)$ and $\pi^*_\beta(\ba|\bs)$ denotes the policy $\pi \in \Pi_{\text{relabel}}$ that maximizes Equation~\ref{eqn:optimize_behavior}. 
\end{proposition}
\begin{proof}
To prove this proposition, we shall quantify the lower-bound on the improvement in the policy performance due to Equation~\ref{eqn:eqn4_detailed} in the empirical MDP, and the potential drop in policy performance in the original MDP due to sampling error, and combine the terms to obtain our bound. First note that for any given policy $\pi$, and a dataset $\mathcal{D}^\mathrm{eff}_i$ with effective behavior policy $\pi_\beta(\ba|\bs)$, the following bound holds~\citep{kumar2020conservative}:
\begin{equation}
\label{eqn:sampling_error}
    J(\pi) \geq J_{\mathcal{D}^\mathrm{eff}_i} (\pi) - \mathcal{O}\left(\frac{1}{(1 - \gamma)^2}\right) \mathbb{E}_{\bs \sim d^{\pi}_{\mathcal{D}^\mathrm{eff}_i}}\left[\sqrt{\frac{D_{\text{CQL}}(\pi, \pi^*_\beta)(\bs) + 1}{|\mathcal{D}^\mathrm{eff}_i(\bs)|}} \right],
\end{equation}
where the $\mathcal{O}(\cdot)$ notation hides constants depending upon the concentration properties of the MDP~\citep{laroche2019safe} and $1 - \delta$, the probability with which the statement holds. Next, we provide guarantees on policy improvement in the empirical MDP. To see this, note that the following statements on $f(\pi_1, \pi_2; \mathcal{D})$ are true:
\begin{align}
    &\forall \pi' \in \Pi_{\text{relabel}}, ~~f(\pi^*, \pi^*_\beta; \mathcal{D}^\mathrm{eff}_i) \geq f(\pi', \pi', \mathcal{D}^\mathrm{eff}_i)\\
    \implies & \forall \pi' \in \Pi_{\text{relabel}},~~ J_{\mathcal{D}^\mathrm{eff}_i}(\pi^*) - \alpha D(\pi^*, \pi^*_\beta) \geq J_{\mathcal{D}^\mathrm{eff}_i}(\pi').
    \label{eqn:inequality_policy}
\end{align}
And additionally, we obtain:
\begin{align}
    &\forall \pi' \in \Pi_\text{relabel}, ~~f(\pi^*, \pi^*_\beta; \mathcal{D}^\mathrm{eff}_i) \geq f(\pi^*, \pi'; \mathcal{D}^\mathrm{eff}_i),\\
    \implies &\forall \pi' \in \Pi_\text{relabel},~~ D(\pi^*, \pi^*_\beta) \leq D(\pi^*, \pi').
\end{align}
Utilizing \ref{eqn:inequality_policy}, we obtain that:
\begin{equation}
\label{eqn:perf_increase}
   J_{\mathcal{D}^\mathrm{eff}_i}(\pi^*) - J_{\mathcal{D}^\mathrm{eff}_i}(\pi_\beta) \geq \alpha D(\pi^*, \pi^*_\beta) + \left( J_{\mathcal{D}^\mathrm{eff}_i}(\pi_\beta^*) - J_{\mathcal{D}^\mathrm{eff}_i}(\pi_\beta) \right) \approx \alpha D(\pi^*, \pi^*_\beta) + \left(J(\pi^*_\beta) - J(\pi_\beta) \right),
\end{equation}
where $\approx$ ignores sampling error terms that do not depend on distributional shift measures like $D_\text{CQL}$ because $\pi^*_\beta$ and $\pi_\beta$ are behavior policies which generated the complete and part of the dataset, and hence these terms are dominated by and subsumed into the sampling error for $\pi^*$. Combining Equations~\ref{eqn:sampling_error} (by setting $\pi = \pi^*$) and \ref{eqn:perf_increase}, we obtain the following safe-policy improvement guarantee for $\pi^*$: $J(\pi^*) - J(\pi_\beta) \geq \zeta$, where $\zeta$ is given by:
\begin{equation*}
    \zeta = \mathcal{O}\left(\frac{1}{(1 - \gamma)^2}\right) \mathbb{E}_{\bs \sim d^{\pi^*}_{\mathcal{D}^\mathrm{eff}_i}}\left[\sqrt{\frac{D_{\text{CQL}}(\pi^*, \pi_\beta^*)(\bs) + 1}{|\mathcal{D}^\mathrm{eff}_i(\bs)|}} \right] -  \left[\alpha D(\pi^*, \pi_\beta^*) + \underbrace{J(\pi^*_\beta) - J(\pi_\beta)}_{\text{(a)}} \right],
\end{equation*}
which proves the desired result.
\end{proof}
Proposition~\ref{prop:spi} indicates that when optimizing the behavior policy with Equation~\ref{eqn:optimize_behavior}, we can improve upon the conventional safe-policy improvement guarantee (Equation~\ref{eqn:single_task_guarantee}) with standard single-task offline RL: not only do we improve via $D_\text{CQL}(\pi^*, \pi_\beta^*)$, since, $D_\text{CQL}(\pi^*, \pi_\beta^*) \leq D_\text{CQL}(\pi^*, \pi_\beta)$, which reduces sampling error, but utilizing this policy $\pi^*_\beta$ also allows us to improve on term $(a)$, since Equation~\ref{eqn:eqn4_detailed} optimizes the behavior policy to be close to the learned policy $\pi^*$ and maximizes the learned policy return $J_{\mathcal{D}^\mathrm{eff}_i}(\pi^*)$ on the effective dataset, thus providing us with a high lower bound on $J(\pi^*_\beta)$. We formalize this insight as Lemma~\ref{lemma:a_gt_0} below:

\begin{lemma}
\label{lemma:a_gt_0}
For sufficiently large $\alpha$, $J_{\mathcal{D}^\mathrm{eff}_i}(\pi^*_\beta) \geq J_{\mathcal{D}^\mathrm{eff}_i}(\pi_\beta)$ and thus $(a) \geq 0$.
\end{lemma}
\begin{proof}
To prove this, we note that using standard difference of returns of two policies, we get the following inequality: $J_{\mathcal{D}^\mathrm{eff}_i}(\pi^*_\beta) \geq J_{\mathcal{D}^\mathrm{eff}_i}(\pi^*) - C \frac{R_{\max}}{1 - \gamma} D_{\mathrm{TV}}(\pi^*, \pi^*_\beta)$. Moreover, from Equation~\ref{eqn:inequality_policy}, we obtain that: $J_{\mathcal{D}^\mathrm{eff}_i}(\pi^*) - \alpha D(\pi^*, \pi^*_\beta) \geq J_{\mathcal{D}^\mathrm{eff}_i}(\pi_\beta)$. So, if $\alpha$ is chosen such that:
\begin{equation}
    \frac{C R_{\max}}{1 - \gamma } D_\mathrm{TV}(\pi^*, \pi_\beta^*) \leq \alpha D(\pi^*, \pi^*_\beta),
\end{equation}
we find that:
\begin{equation*}
    J_{\mathcal{D}^\mathrm{eff}_i}(\pi^*_\beta) \geq J_{\mathcal{D}^\mathrm{eff}_i}(\pi^*) - C \frac{R_{\max}}{1 - \gamma} D_{\mathrm{TV}}(\pi^*, \pi^*_\beta) \geq J_{\mathcal{D}^\mathrm{eff}_i}(\pi^*) - \alpha D(\pi^*, \pi^*_\beta) \geq J_{\mathcal{D}^\mathrm{eff}_i}(\pi_\beta),
\end{equation*}
implying that $(a) \geq 0$. For the edge cases when either $D_\mathrm{TV}(\pi^*, \pi^*_\beta) = 0$ or $D_{\text{CQL}}(\pi^*, \pi^*_\beta) = 0$, we note that $\pi^*(\ba|\bs) = \pi^*_\beta(\ba|\bs)$, which trivially implies that $J_{\mathcal{D}^\mathrm{eff}_i}(\pi^*_\beta) = J_{\mathcal{D}^\mathrm{eff}_i}(\pi^*) \geq J_{\mathcal{D}^\mathrm{eff}_i}(\pi_\beta)$, because $\pi^*$ improves over $\pi_\beta$ on the dataset. Thus, term $(a)$ is positive for large-enough $\alpha$ and the bound in Proposition~\ref{prop:spi} gains from this term additionally.
\end{proof}

Finally, we show that the sampling error term is controlled when utilizing Equation~\ref{eqn:optimize_behavior}. We will show in Lemma~\ref{lemma:sampling_error} that the sampling error in Proposition~\ref{prop:spi} is controlled to be not much bigger than the error just due to variance, since distributional shift is bounded with Equation~\ref{eqn:optimize_behavior}.
\begin{lemma}
\label{lemma:sampling_error}
If $\pi^*$ and $\pi^*_\beta$ obtained from Equation~\ref{eqn:optimize_behavior} satisfy, $D_\text{CQL}(\pi^*, \pi^*_\beta) \leq \varepsilon \ll 1$, then:
\begin{equation}
    (\$) := \mathbb{E}_{\bs \sim d^{\pi^*}_{\mathcal{D}^\mathrm{eff}_i}}\left[\sqrt{\frac{D_{\text{CQL}}(\pi^*, \pi_\beta^*)(\bs) + 1}{|\mathcal{D}^\mathrm{eff}_i(\bs)|}} \right] \leq (1 + \varepsilon)^{\frac{1}{2}} \underbrace{\mathbb{E}_{\bs \sim d^{\pi^*}_{\mathcal{D}^\mathrm{eff}_i}}\left[\sqrt{\frac{1}{|\mathcal{D}^\mathrm{eff}_i(\bs)|}}~ \right]}_{:= \text{sampling error w/o distribution shift}}.
\end{equation}
\end{lemma}
\begin{proof}
This lemma can be proved via a simple application of the Cauchy-Schwarz inequality. We can partition the first term as a sum over dot products of two vectors such that:
\begin{align*}
    (\$) &= \sum_{\bs} \sqrt{d^{\pi^*}_{\mathcal{D}^\mathrm{eff}_i}(\bs) (D_{\text{CQL}}(\pi^*, \pi_\beta^*)(\bs) + 1)} \sqrt{\frac{d^{\pi^*}_{\mathcal{D}^\mathrm{eff}_i}(\bs)}{|\mathcal{D}^\mathrm{eff}_i(\bs)|}}\\
    &\leq \sqrt{\left( \sum_{\bs} d^{\pi^*}_{\mathcal{D}^\mathrm{eff}_i}(\bs) (D_{\text{CQL}}(\pi^*, \pi_\beta^*)(\bs) + 1) \right) \cdot \left( \sum_{\bs} \frac{d^{\pi^*}_{\mathcal{D}^\mathrm{eff}_i}(\bs)}{|\mathcal{D}^\mathrm{eff}_i(\bs)|} \right)}\\
    &= \sqrt{\mathbb{E}_{\bs \sim d^{\pi^*}_{\mathcal{D}^\mathrm{eff}_i}}\left[D_{\text{CQL}}(\pi^*, \pi_\beta^*)(\bs) + 1\right] \mathbb{E}_{\bs \sim d^{\pi^*}_{\mathcal{D}^\mathrm{eff}_i}}\left[ \frac{1}{|\mathcal{D}^\mathrm{eff}_i(\bs)|} \right]} \leq (1 + \varepsilon)^{0.5} \mathbb{E}_{\bs \sim d^{\pi^*}_{\mathcal{D}^\mathrm{eff}_i}}\left[\sqrt{\frac{1}{|\mathcal{D}^\mathrm{eff}_i(\bs)|}}~ \right],
\end{align*}
where we note that $\mathbb{E}_{\bs \sim d^{\pi^*}_{\mathcal{D}^\mathrm{eff}_i}}\left[D_{\text{CQL}}(\pi^*, \pi_\beta^*)(\bs)\right] = D_\text{CQL}(\pi^*, \pi^*_\beta) \leq \varepsilon$ (based on the given information in the Lemma) and that $\sqrt{\sum_i w_i \frac{1}{x_i}} \leq \sum_i w_i \frac{1}{\sqrt{x_i}}$ for $x_i, w_i > 0$ and $\sum_i w_i = 1$, via Jensen's inequality for concave functions.
\end{proof}

\textbf{To summarize,} combining Lemmas~\ref{lemma:a_gt_0} and \ref{lemma:sampling_error} with Proposition~\ref{prop:spi}, we conclude that utilizing Equation~\ref{eqn:optimize_behavior} controls the increase in sampling error due to distributional shift, and provides improvement guarantees on the learned policy beyond the behavior policy of the original dataset. We also briefly now discuss the comparison between \methodname\ and complete data sharing. Complete data sharing would try to reduce sampling error by increasing $|\mathcal{D}^\mathrm{eff}_i(\bs)|$, but then it can also increase distributional shift, $D_\text{CQL}(\pi^*, \pi^*_\beta)$ as discussed in Section~\ref{sec:analysis}. On the other hand, CDS increases the dataset size while also controlling for distributional shift (as we discussed in the analysis above), making it enjoy the benefits of complete data sharing and avoiding its pitfalls, intuitively. On the other hand, no data sharing will just incur high sampling error due to limited dataset size. 

\subsection{From Equation~\ref{eqn:optimize_behavior} to Practical \methodname\ (Equation~\ref{eqn:method})}
\label{sec:practical_cds}
The goal of our practical algorithm is to convert Equation~\ref{eqn:optimize_behavior} to a practical algorithm while retaining the policy improvement guarantees derived in Proposition~\ref{prop:spi}. Since our algorithm does not utilize any estimator for dataset counts $|\mathcal{D}^\mathrm{eff}_i(\bs)|$, and since we operate in a continuous state-action space, our goal is to retain the guarantees of increased return of $\pi^*_\beta$, while also avoiding sampling error. 

With this goal, we first need to relax the state-distribution in Equation~\ref{eqn:optimize_behavior}: while both $J_{\mathcal{D}^\mathrm{eff}_i}(\pi)$ and $D_\text{CQL}(\pi, \pi_\beta)$ are computed as expectations under the marginal state-distribution of policy $\pi(\ba|\bs)$ on the MDP defined by the dataset $\mathcal{D}^\mathrm{eff}_i$, for deriving a practical method we relax the state distribution to use the dataset state-distribution $d^{\pi^*_\beta}$ and rewrite the objective in accordance with most practical implementations of actor-critic algorithms~\citep{degris2012off,Abdolmaleki2018MaximumAP,haarnoja2018soft,fujimoto2018addressing,lillicrap2015continuous} below:
\begin{equation}
\label{eqn:practical_eqn4}
    \text{(Practical Equation~\ref{eqn:optimize_behavior})}~~~~ \max_{\pi} \max_{\pi_\beta \in \Pi_\text{relabel}}~~ \mathbb{E}_{\bs \sim \mathcal{D}^{\mathrm{eff}}_i}[\mathbb{E}_{\ba \sim \pi(\ba|\bs)}[Q(\bs, \ba)] - \alpha D(\pi(\cdot|\bs), \pi_\beta(\cdot|\bs))]
\end{equation}
This practical approximation in Equation~\ref{eqn:practical_eqn4} is even more justified with conservative RL algorithms when a large $\alpha$ is used, since a larger $\alpha$ implies a smaller value for $D(\pi^*, \pi_\beta^*)$ found by Equation~\ref{eqn:optimize_behavior}, which in turn means that state-distributions $d^{\pi_\beta^*}$ and $d^{\pi^*}$ are close to each other~\citep{schulman2015trust}. Thus, our policy improvement objective optimizes the policies $\pi$ and $\pi_\beta$ by maximizing the conservative Q-function: $\hat{Q}^\pi(\bs, \ba) = Q(\bs, \ba) - \alpha \left(\frac{\pi(\ba|\bs)}{\pi_\beta(\ba|\bs)} - 1 \right)$, that appears inside the expectation in Equation~\ref{eqn:practical_eqn4}. While optimizing the policy $\pi$ with respect to this conservative Q-function $\hat{Q}^\pi(\bs, \ba)$ is equivalent to a standard policy improvement update utilized by most actor-critic methods~\citep{fujimoto2018addressing,haarnoja2018soft,kumar2020conservative}, we can optimize $\hat{Q}^\pi(\bs, \ba)$ with respect to $\pi_\beta \in \Pi_{\text{relabel}}$ by relabeling only those transitions $(\bs, \ba, r'_i, \bs') \in \mathcal{D}_{j \rightarrow i}$ that increase the expected conservative Q-value $\mathbb{E}_{\bs \sim \mathcal{D}^\mathrm{eff}_i}\left[\mathbb{E}_{\ba \sim \pi_\beta(\cdot|\bs)}\left[\hat{Q}^\pi(\bs, \ba)\right] \right]$. Note that we relaxed the expectation $\ba \sim \pi(\ba|\bs)$ to $\ba \sim \pi_\beta(\ba|\bs)$ in this expectation, which can be done upto a lower-bound of the objective in Equation~\ref{eqn:practical_eqn4} for a large $\alpha$, since the resulting policies $\pi$ and $\pi_\beta$ are close to each other.  

The last step in our practical algorithm is to modify the solution of Equation~\ref{eqn:practical_eqn4} to still retain the benefits of reduced sampling error as discussed in Proposition~\ref{prop:spi}. To do so, we want to relabel as many points as possible, thus increasing $|\mathcal{D}^\mathrm{eff}_i(\bs)|$, which leads to reduced sampling error. Since quantifying $|\mathcal{D}^\mathrm{eff}_i(\bs)|$ in continuous state-action spaces will require additional machinery such as density-models, we avoid these for the sake of simplicity, and instead choose to relabel every datapoint $(\bs, \ba) \in \mathcal{D}_{j \rightarrow i}$ that satisfies $Q^\pi(\bs, \ba; i) \geq \mathbb{E}_{\bs, \ba \sim \mathcal{D}_i}[\hat{Q}^\pi(\bs, \ba; i)] \geq 0$ to task $i$. These datapoints definitely increase the conservative Q-value and hence increase the objective in Equation~\ref{eqn:practical_eqn4} (though do not globally maximize it), while also enjoying properties of reduced sampling error (Proposition~\ref{prop:spi}). This discussion motivates our practical algorithm in Equation~\ref{eqn:method}.  

\section{Experimental details}
\label{app:details}

In this section, we provide the training details of CDS in Appendix~\ref{app:training_details} and also include the details on the environment and datasets that we use for the evaluation in Appendix~\ref{app:env_data_details}. Finally, we include the discussion on the compute information in Appendix~\ref{app:compute_details}. \arxiv{We also compare CDS to an offline RL with pretrained representations from multi-task datasets method~\citep{yang2021representation}.}

\subsection{Training details}
\label{app:training_details}
Our practical implementation of \methodname\ optimizes the following objectives for training the critic and the policy:
\vspace*{-5pt}
\begin{small}
\begin{align*}
    \hat{Q}^{k+1} \leftarrow& \arg\min_{\hat{Q}} \mathbb{E}_{i\sim[N]}\left[\beta\left(\mathbb{E}_{j \sim[N]}\left[\mathbb{E}_{\bs \sim \mathcal{D}_j, \ba \sim \mu(\cdot|\bs,i)}\left[w_{\mathrm{\methodname}}(\bs, \ba; j \rightarrow i)\hat{Q}(\bs,\ba,i)\right]\right.\right.\right.\\
    &\left.\left.\left.- \mathbb{E}_{\bs, \ba \sim \mathcal{D}_j}\left[w_{\mathrm{\methodname}}(\bs, \ba; j \rightarrow i)\hat{Q}(\bs,\ba, i)\right]\right]\right)\right.\\
    &\left. + \frac{1}{2}\mathbb{E}_{j\sim[N],(\bs, \ba, \bs') \sim \mathcal{D}_j}\left[w_{\mathrm{\methodname}}(\bs, \ba; j \rightarrow i) \left(\hat{Q}(\bs, \ba, i) - \widehat{\bellman}^\policy \hat{Q}^k(\bs, \ba, i)\right)^2 \right]\right],
\end{align*}
\end{small}
\vspace*{-19pt}
and
\[
\policy \leftarrow \arg \max_{\policy'} \ \mathbb{E}_{i \sim [N]}\left[\mathbb{E}_{j\sim[N],\bs \sim \mathcal{D}_j, \ba \sim \policy'(\cdot | \bs, i)} \left[ w_{\mathrm{\methodname}}(\bs, \ba; j \rightarrow i)\hat{Q}^\policy (\bs, \ba, i) \right]\right],
\]
where $\beta$ is the coefficient of the CQL penalty on distribution shift, $\mu$ is a wide sampling distribution as in CQL and $\widehat{\bellman}$ is the sample-based Bellman operator.

To compute the relabeling weight $w_{\mathrm{\methodname}}(\bs, \ba; j \rightarrow i) := \sigma \left(\frac{\Delta(\bs, \ba; j \rightarrow i)}{\tau} \right)$, we need to pick the value of the temperature term $\tau$. Instead of tuning $\tau$ manually, we follow the the adaptive temperature scaling scheme from \citep{kumar2020discor}. Specifically, we compute an exponential running average of $\Delta(\bs, \ba; j \rightarrow i)$ with decay $0.995$ for each task and use it as $\tau$. We additionally clip the adaptive temperature term with a minimum and maximum threshold, which we tune manually. For multi-task halfcheetah, walker2d and ant, we clip the adaptive temperature such that it lies within $[10,\infty]$, $[5,\infty]$ and $[10,25]$ respectively. For the multi-task Meta-Wold experiment, we use $[1, 50]$ for the clipping. For multi-task Antmaze, we used a range of $[10, \infty]$ for all the domains. We do not clip the temperature term on vision-based domains.

For state-based experiments, we use a stratified batch with $128$ transitions for each task for the critic and policy learning. For each task $i$, we sample $64$ transitions from $\mathcal{D}_i$ and another $64$ transitions from $\cup_{j \neq i} \mathcal{D}_{j \rightarrow i}$, i.e. the relabeled datasets of all the other tasks. When computing $\Delta(\bs, \ba; j \rightarrow i)$, we only apply the weight to relabeled data on multi-task Meta-World environments and multi-task vision-based robotic manipulation tasks while also applying the weight to the original data drawn from $\mathcal{D}_i$ with 50\% chance for each task $i\in[N]$ in the remaining domains.

We use CQL~\citep{kumar2020conservative} as the base offline RL algorithm. On state-based experiments, we mostly follow the hyperparameters provided in prior work~\citep{kumar2020conservative}. One exception is that on the multi-task ant domain, we set $\beta = 5.0$ and on the other two locomotion environments and the multi-task Meta-World domain, we use $\beta = 1.0$. On multi-task AntMaze, we use the Lagrange version of CQL, where the multiplier $\beta$ is automatically tuned against a pre-specific constraint value on the CQL loss equal to $\tau = 5.0$. We use a policy learning rate $1e-4$ and a critic learning rate $3e-4$ as in \citep{kumar2020conservative}. On the vision-based environment, instead of using the direct CQL algorithm, we follow \citep{chebotar2021actionable} and sample unseen actions according to the soft-max distritbution of the Q-values and set its Q target value to $0$. This algorithm can be viewed the version of CQL with $\beta=1.0$ in Eq.1 in \citep{kumar2020conservative}, i.e. removing the term of negative expected Q-values on the dataset. We follow the other hyperparameters from prior work~\citep{kalashnikov2018scalable,chebotar2021actionable,kalashnikov2021mt}.

For the choice architectures, in the domains with low-dimensional state inputs, we use 3-layer feedforward neural networks with $256$ hidden units for both the Q-networks and the policy. We append a one-hot task vector to the state of each environment. For the vision-based experiment, our Q-network architecture follows from multi-headed convolutional networks used in MT-Opt~\citep{kalashnikov2021mt}. For the observation input, we use images with dimension $472 \times 472 \times 3$ along with additional state features $(g_\text{status}, g_\text{height})$ as well as the one-hot task vector as in \citep{kalashnikov2021mt}. For the action input, we use Cartesian
space control of the end-effector of the robot in 4D space (3D position and azimuth angle) along with two discrete actions for
opening/closing the gripper and terminating the episode respectively. More details can be found in \citep{kalashnikov2018scalable,kalashnikov2021mt}.

\subsection{Environment and dataset details}
\label{app:env_data_details}

In this subsection, we discuss the details of how we set up the multi-task environment and how we collect the offline datasets. We want to acknowledge that all datasets with state inputs use the MIT License.

\textbf{Multi-task locomotion domains.} We construct the environment by changing the reward function in \citep{brockman2016openai}. On the halfcheetah environment, we follow \citep{yu2020mopo} and set the reward functions of task \texttt{run forward}, \texttt{run backward} and \texttt{jump} as $r(s, a) = \max\{v_x, 3\} - 0.1*\|a\|_2^2$, $r(s, a) = -\max\{v_x, 3\} - 0.1*\|a\|_2^2$ and $r(s, a) = - 0.1*\|a\|_2^2 + 15*(z - \text{init z})$ respectively where $v_x$ denotes the velocity along the x-axis and $z$ denotes the z-position of the half-cheetah and $\text{init z}$ denotes the initial z-position. Similarly, on walker2d, the reward functions of the three tasks are $r(s, a) = v_x - 0.001*\|a\|_2^2$, $r(s, a) = -v_x - 0.001*\|a\|_2^2$ and $r(s, a) = - \|v_x\| - 0.001*\|a\|_2^2 + 10*(z - \text{init z})$ respectively. Finally, on ant, the reward functions of the three tasks are $r(s, a) = v_x - 0.5*\|a\|_2^2 - 0.005*\text{contact-cost}$, $r(s, a) = -v_x - 0.5*\|a\|_2^2 - 0.005*\text{contact-cost}$ and $r(s, a) = - \|v_x\| - 0.5*\|a\|_2^2 - 0.005*\text{contact-cost} + 10*(z - \text{init z})$.

On each of the multi-task locomotion environment, we train each task with SAC~\citep{haarnoja2018soft} for 500 epochs. For medium-replay datasets, we take the whole replay buffer after the online SAC is trained for 100 epochs. For medium datasets, we take the online single-task SAC policy after 100 epochs and collect 500 trajectories with the medium-level policy. For expert datasets, we take the final online SAC policy and collect 5 trajectories with it for walker2d and halfcheetah and 20 trajectories for ant.

\textbf{Meta-World domains.} We take the \texttt{door open}, \texttt{door close}, \texttt{drawer open} and \texttt{drawer close} environments from the open-sourced Meta-World~\citep{yu2020metaworld} repo\footnote{The Meta-World environment can be found at the public repo \url{https://github.com/rlworkgroup/metaworld}}. We put both the door and the drawer on the same scene to make sure the state space of all four tasks are shared. For offline training, we use sparse rewards for each task by replacing the dense reward defined in Meta-World with the success condition defined in the public repo. Therefore, each task gets a reward of 1 if the task is fully completed and 0 otherwise.

For generating the offline datasets, we train each task with online SAC using the dense reward defined in Meta-World for 500 epochs. For medium-replay datasets, we take the whole replay buffer of the online SAC until 150 epochs. For the expert datasets, we run the final online SAC policy to collect 10 trajectories.

\textbf{AntMaze domains.} We take the \texttt{antmaze-medium-play} and \texttt{antmaze-large-play} datasets from D4RL~\citep{fu2020d4rl} and convert the datasets into multi-task datasets in two ways. In the undirected version of these tasks, we split the dataset randomly into equal sized partitions, and then assign each partition to a particular randomly chosen task. Thus, the task data observed in the data for each task is largely unsuccessful for the particular task it is assigned to and effective data sharing is essential for obtaining good performance. The second setting is the directed data setting where a trajectory in the dataset is marked to belong to the task corresponding to the actual end goal of the trajectory. A sparse reward equal to +1 is provided to an agent when the current state reaches within a 0.5 radius of the task goal as was used default by \citet{fu2020d4rl}. 

\textbf{Vision-based robotic manipulation domains.} Following MT-Opt~\citep{kalashnikov2021mt}, we use sparse rewards for each task, i.e. reward 1 for success episodes and 0 otherwise. We define successes using the success detectors defined in \citep{kalashnikov2021mt}. 
To collect data for vision-based experiments, we train a policy for each task individually by running QT-Opt~\citep{kalashnikov2018scalable} with default hyperparameters until the task reaches 40\% success rate for picking skills and 80\% success rate for placing skills. We take the whole replay buffer of each task and combine all of such replay buffers to form the multi-task offline dataset with total 100K episodes where each episode has 25 transitions.

\subsection{Computation Complexity}
\label{app:compute_details}

For all the state-based experiments, we train CDS on a single NVIDIA GeForce RTX 2080 Ti for one day. For the image-based robotic manipulation experiments, we train it on 16 TPUs for three days.

\arxiv{\section{Visualizations, Comparisons and Additional Experiments}
\label{app:additional_exp}

In this section, we perform diagnostic and ablation experiments to: \textbf{(1)} understand the efficacy of CDS when applied with other base offline RL algorithms, such as BRAC~\citep{wu2019behavior}, \textbf{(2)} visualize the weights learned by CDS to understand if the weighting scheme induced by CDS corresponds to what we would intuitively expect on different tasks, and \textbf{(3)} compare CDS to a prior approach that performs representation learning from offline multi-task datasets and then runs vanilla multi-task RL algorithm on top of the learned representations. We discuss these experiments next.

\subsection{Applying CDS with BRAC~\citep{wu2019behavior}, A Policy-Constraint Offline RL Algorithm}
\label{app:brac_results}

We implemented CDS on top of BRAC which is different from CQL that penalizes Q-functions. BRAC computes the divergence $D(\pi, \pi_\beta)$ in Equation~\ref{eqn:generic_offline_rl} explicitly and penalizes the reward function $r(\bs, \ba)$ with this value in the Bellman backups. To apply CDS to BRAC, we need to compute a conservative estimate of the Q-value as discussed in Section~\ref{sec:complete_cds}. While the Q-function from CQL directly provides us with this conservative estimate, BRAC does not directly learn a conservative Q-function estimator. Therefore,  for BRAC, we compute this conservative estimate by explicitly subtracting KL divergence between the learned policy $\pi(a|s)$ and the behavior policy $\pi^\beta$ on state-action tuples $(s,a)$ from the learned Q-function’s prediction. Formally, this means that we utilize $\hat{Q}(s, a) := Q(s,a) - \alpha \mathrm{D}_\text{KL}(\pi(a|s),\pi^\beta(a|s))$ as our conservative Q-value estimate for BRAC. Given these conservative Q-value estimate, CDS weights can be computed directly using Equation~\ref{eqn:method}. 

\begin{table*}[h]
\centering
\vspace*{0.1cm}
\scriptsize
\resizebox{\textwidth}{!}{\begin{tabular}{l|l|r|r|r}
\toprule
\textbf{Environment} & \textbf{Tasks / Dataset type} & \textbf{BRAC + \methodname\ (ours)} & \textbf{BRAC + No Sharing} & \textbf{BRAC + Sharing All}\\ \midrule
& door open / medium-replay & \textbf{44.0\%}$\pm$3.0\% & 35.0\%$\pm$25.9\% & 38.0\%$\pm$2.2\%\\
& door close / expert & \textbf{32.5\%} $\pm$ 5.0\% & 5.0\% $\pm$ 8.6\% & 8.6\% $\pm$ 3.4\% \\
Meta-World~\citep{yu2020metaworld}& drawer open / expert & \textbf{28.5\%}$\pm$3.5\% & 21.8\%$\pm$5.6\% & 0.0\% $\pm$ 0.0\%\\
& drawer close / medium-replay & 100\%$\pm$0.0\% & 100.0\%$\pm$0.0\% & 99.0\%$\pm$0.7\%\\
& \CC \textbf{average} & \CC \textbf{52.5\%}$\pm$7.4\% & \CC 22.5\%$\pm$13.3\% & \CC 40.0\%$\pm$5.0\%\\
\bottomrule
\end{tabular}}
\vspace{-0.2cm}
\caption{\footnotesize \arxiv{Applying CDS on top of BRAC. Note that CDS + BRAC imrpoves over both BRAC + Sharing All and BRAC + No sharing, indicating that CDS is effective over other offline RL algorithms such as BRAC as well. The $\pm$ values indicate the value of the standard deviation of runs, and results are averaged over three seeds.}
}
\label{tbl:brac_comparison}
\normalsize
\vspace{-0.3cm}
\end{table*}

We evaluated BRAC + CDS on the Meta-World tasks and compared it to BRAC + Sharing All and BRAC + No Sharing. We present the results in Table~\ref{tbl:brac_comparison}. We use $\pm$ to denote the 95\%-confidence interval. As observed below, BRAC + CDS significantly outperforms BRAC with Sharing All and BRAC with No sharing. This indicates that CDS is effective on top of BRAC. 

\subsection{Analyzing CDS weights for Different Scenarios} 
\label{app:cds_vis}

Next, to understand if the weights assigned by CDS align with our expectation for which transitions should be shared between tasks, we perform diagnostic analysis studies on the weights learned by CDS on the Meta-World and Antmaze domains.

\textbf{On the Meta-World environment}, we would expect that for a given target task, say Drawer Close, transitions from a task that involves a different object (door) and a different skill (open) would not be as useful for learning. To understand if CDS weights reflect this expectation, we compare the average CDS weights on transitions from all the other tasks to two target tasks, Door Open and Drawer Close, respectively and present the results in Table~\ref{tbl:cds_weights}. We sort the CDS weights in the descending order. As shown, indeed CDS assigns higher weights to more related tasks and thus shares data from those tasks. In particular, the CDS weights for relabeling data from the task that handles the same object as the target task are much higher than the weights for tasks that consider a different object. 

\begin{table}[h]
\centering
\footnotesize
\begin{tabular}{l|r}
\toprule
{\textbf{Relabeling Direction}}                 & \textbf{CDS weight} \\
\midrule
{{door close $\rightarrow$ door open}}   & 0.46       \\
{{drawer open $\rightarrow$ door open}}  & 0.10       \\
{{Drawer close $\rightarrow$ door open}} & 0.02       \\ \midrule
{drawer open $\rightarrow$ drawer close}                     & 0.35       \\
{door open $\rightarrow$ drawer close}                       & 0.26       \\
{door close $\rightarrow$ drawer close}                      & 0.22      \\ 
\bottomrule
\end{tabular}
\vspace{0.2cm}
\caption{\footnotesize \arxiv{On the Meta-World domain, we visualize the CDS weights of data relabeled from other tasks to the two target tasks door open and drawer close shown in the second row and third row respectively. We sort the CDS weights for relabeled tasks to a particular target task in the descending order. As shown in the table, CDS upweights tasks that are more related to the target task, e.g. manipulating the same object.}
}
\vspace{-0.3cm}
\label{tbl:cds_weights}
\end{table}

For example, when relabeling to the target task Door Open, datapoints from task Door Close are assigned with much higher weights than those from either task Drawer Open or task Drawer Close. This suggests that CDS filters the irrelevant transitions for learning a given task.

\textbf{On the AntMaze-large environment}, with undirected data, we visualize the CDS weight for the various tasks (goals) in the form of a heatmap and present the results in Figure~\ref{fig:antmaze_vis}. To generate this plot, we sample a set of state-action pairs from the entire dataset for all tasks, and then plot the weights assigned by CDS as the color of the point marker at the $(x, y)$ locations of these state-action pairs in the maze. Each plot computes the CDS weight corresponding to the target task (goal) indicated by the red $\times$ in the plot. As can be seen in Figure~\ref{fig:antmaze_vis}, CDS assigns higher weights to transitions from nearby goals as compared to transitions from farther away goals. This matches our expectation: transitions from nearby $(x, y)$ locations are likely to be the most useful in learning a particular target task and CDS chooses to share these transitions to the target task.}

\arxiv{\subsection{Comparison of CDS with Other Alternatives to Data Sharing: Utilizing Multi-Task Datasets for Learning Pre-Trained Representations} 
\label{app:pretrain_comparison}

\begin{figure}[t]
    \centering
    \includegraphics[width=0.75\textwidth]{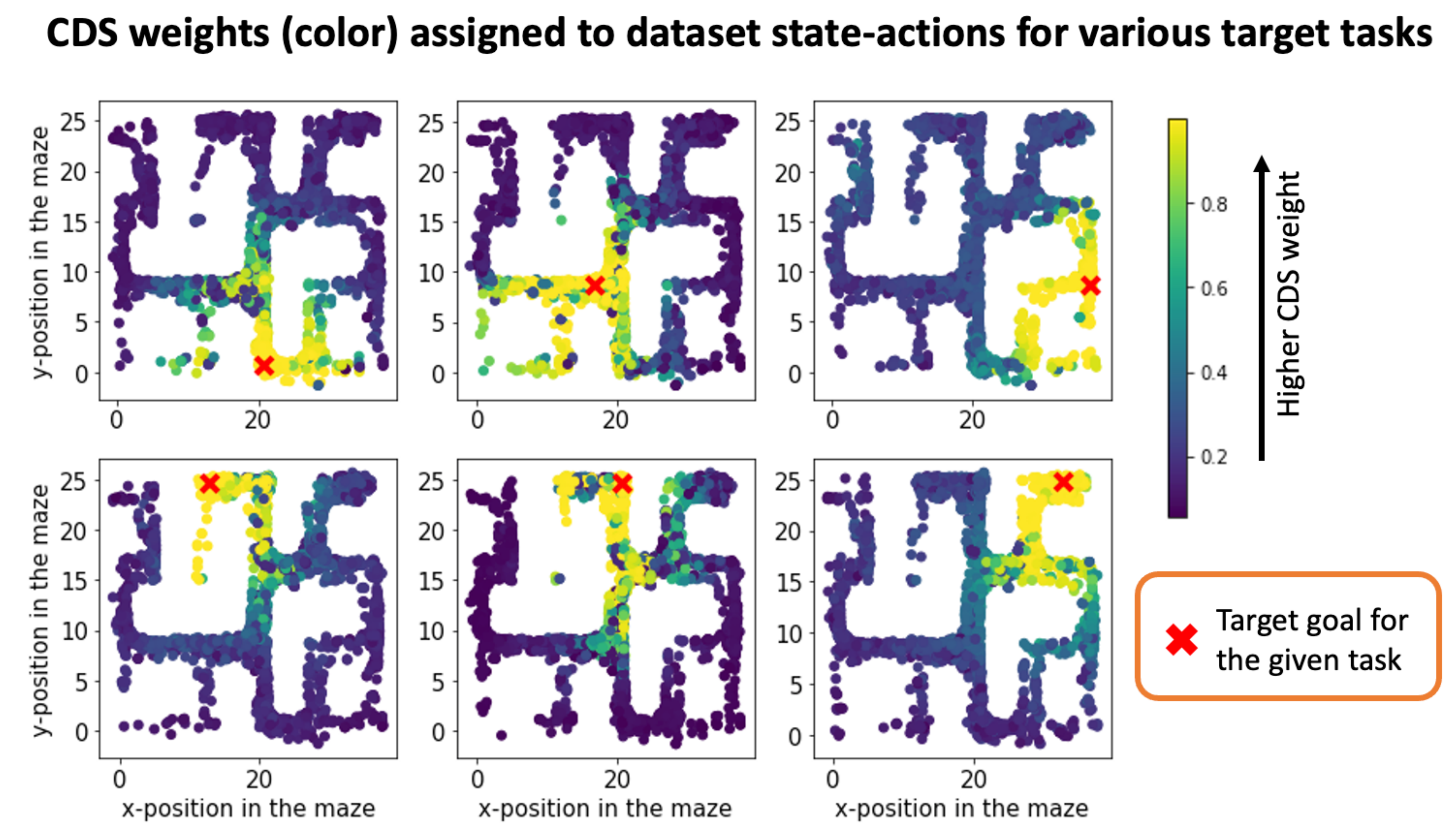}
    \vspace{-0.2cm}
    \caption{\footnotesize A visualization of the weights assigned by \methodname\ to various transitions in the antmaze dataset for six target goals (indicated by \textcolor{red}{$\times$} clustered by their spatial location. Note that CDS up-weights transitions spatially close to the target goal (indicated in the brighter yellow color), matching our expectation.}
    \label{fig:antmaze_vis}
    \vspace{-0.4cm}
\end{figure}

Finally, we aim to empirically verify how other alternatives to data sharing perform on multi-task offline RL problems. One simple approach to utilize data from other tasks is to use this data to learn low-dimensional representations that capture meaningful information about the environment initially in a pre-training phase and then utilize these representations for improved multi-task RL without any specialized data sharing schemes. To assess the efficacy of this alternate approach of using multi-task offline data, in Table~\ref{tbl:pretrain_comparison}, we performed an experiment on the Meta-World domain that first utilizes the data from all the tasks to learn a shared representation using the best method, ACL~\citep{yang2021representation} and then runs standard offline multi-task RL on top of this representation. We denote the method as \textbf{Offline Pretraining}. We include the average task success rates of all tasks in the table below. While the representation learning approach improves over standard multi-task RL without representation learning (\textbf{No Sharing}) consistent with the findings in \citep{yang2021representation}, we still find that CDS with no representation learning outperforms this representation learning approach by a large margin on multi-task performance, which suggests that conservative data sharing is more important than pure pretrained representation from multi-task datasets in the offline multi-task setting. We finally remark that in principle, we could also utilize representation learning approaches in conjunction with data sharing strategies and systematically characterizing this class of hybrid approaches is a topic of future work.}

\begin{table*}[t!]
\centering
\vspace*{0.1cm}
\scriptsize
\resizebox{\textwidth}{!}{\begin{tabular}{l|l|r|r|r}
\toprule
\textbf{Environment} & \textbf{Tasks / Dataset type} & \textbf{\methodname\ (ours)} & \textbf{No Sharing} & \textbf{Offline Petraining~\citep{yang2021representation}}\\ \midrule
& door open / medium-replay & \textbf{58.4\%}$\pm$9.3\% & 4.0\%$\pm$6.1\% & 48.0\%$\pm$40.0\%\\
Meta-World~\citep{yu2020metaworld}& drawer open / expert & \textbf{57.9\%}$\pm$16.2\% & 16.0\%$\pm$17.5\% & 1.3\% $\pm$ 1.4\%\\
& drawer close / medium-replay & 98.8\%$\pm$0.7\% & 99.0\%$\pm$0.7\% & 96.0\%$\pm$0.9\%\\
& \CC \textbf{average} & \CC \textbf{70.1\%}$\pm$8.1\% & \CC 33.4\%$\pm$8.3\% & \CC 38.7\%$\pm$11.1\%\\
\bottomrule
\end{tabular}}
\vspace{-0.2cm}
\caption{\footnotesize \arxiv{Comparison between CDS and Offline Pretraining~\citep{yang2021representation} that pretrains the representation from the multi-task offline data and then runs multi-task offline RL on top of the learned representation on the Meta-World domain. Numbers are averaged across 6 seeds, $\pm$ the 95$\%$-confidence interval. \methodname\ significantly outperforms Offline Pretraining.}
}
\label{tbl:pretrain_comparison}
\normalsize
\vspace{-0.3cm}
\end{table*}

\end{document}